\documentclass[twoside,11pt]{article}

\usepackage[abbrvbib, preprint]{jmlr2e}
\usepackage{amsmath}
\usepackage{amssymb}
\usepackage{amsfonts}
\usepackage{mathtools}
\usepackage{bm}

\usepackage[table,xcdraw]{xcolor}
\definecolor{cambridgeblue}{RGB}{163, 193, 173}
\definecolor{CambridgeSuperLightBlue}{RGB}{204, 227, 245}
\definecolor{CambridgeSuperLightOrange}{RGB}{252, 229, 205}
\definecolor{CambridgeSuperLightPurple}{RGB}{142,37,141}
\definecolor{CambridgeSuperLightGreen}{RGB}{217, 234, 211}
\definecolor{CambridgeDarkLightBlue}{RGB}{84,138,182}
\definecolor{CambridgeLightBlue}{RGB}{106,173,228}
\definecolor{CambridgeLightOrange}{RGB}{239,189,71}
\definecolor{CambridgeLightGreen}{RGB}{168,180,0}
\definecolor{CambridgeLightPurple}{RGB}{181,147,155}
\definecolor{CambridgeLightTeal}{RGB}{163,193,173}
\definecolor{CambridgeCoreBlue}{RGB}{0,115,207}
\definecolor{CambridgeCoreOrange}{RGB}{227,114,34}
\definecolor{CambridgeCoreBurgundy}{RGB}{144,27,59}
\definecolor{CambridgeCoreGreen}{RGB}{88,166,24}
\definecolor{CambridgeCorePurple}{RGB}{142,37,141}
\definecolor{CambridgeCoreTeal}{RGB}{0,179,190}
\definecolor{CambridgeDarkBlue}{RGB}{0,62,114}
\definecolor{CambridgeDarkOrange}{RGB}{200,78,0}
\definecolor{CambridgeDarkGreen}{RGB}{67,81,37}
\definecolor{CambridgeDarkPurple}{RGB}{65,45,93}
\definecolor{CambridgeDarkTeal}{RGB}{21,101,112}
\definecolor{CambridgeWhitePrint}{RGB}{255,255,255}
\definecolor{CambridgeBlackPrint}{RGB}{0,0,0}
\definecolor{CambridgeYellow}{RGB}{220,20,6}

\definecolor{mygrey}{RGB}{170, 170, 170}
\definecolor{mylightgrey}{RGB}{230, 230, 230}

\usepackage{tikz}
\usetikzlibrary{automata, positioning, arrows.meta, decorations.markings, calc}
\tikzset{
    middle arrow/.style={
        decoration={markings,
            mark=at position #1 with {\arrow{>}}}, % Choose arrow style here
        postaction={decorate},
        thick
    }
}

\usepackage{aliascnt}

\usepackage{subcaption}
\usepackage{caption}
\usepackage{pgfplots}
\usepgfplotslibrary{groupplots}
\pgfplotsset{compat=1.18}

\usepackage{float}

\usepackage{enumitem}
\newaliascnt{definition}{theorem}
\newtheorem{definition}[definition]{Definition}
\aliascntresetthe{definition}

\newaliascnt{assumption}{theorem}
\newtheorem{assumption}[assumption]{Assumption}
\aliascntresetthe{assumption}

\newaliascnt{lemma}{theorem}
\newtheorem{lemma}[lemma]{Lemma}
\aliascntresetthe{lemma}

\newaliascnt{proposition}{theorem}
\newtheorem{proposition}[proposition]{Proposition}
\aliascntresetthe{proposition}

\newcommand{\xset}{\mathcal{X}}

\newcommand{\qset}{\mathcal{Q}}

\newcommand{\aset}{\mathcal{A}}
\newcommand{\sset}{\mathcal{S}}

\newcommand{\polset}{\mathcal{P}}

\newcommand{\real}{\mathbb{R}}

\newcommand{\region}{\graval}
\newcommand{\pol}{\pi}
\newcommand{\update}{\mu}
\newcommand{\start}{\sigma}

\newcommand{\qval}{q}

\newcommand{\qopt}{\qval_{\best{\pol}}}

\newcommand{\grpol}{P}
\newcommand{\graval}{Q}

\newcommand{\discount}{\gamma}
\newcommand{\lrate}{\alpha}

\newcommand{\best}[1]{{#1}_*}

\newcommand{\kpo}{{{k+1}}}

\newcommand{\nth}{{\textup{th}}}

\usepackage{lastpage}
\jmlrheading{23}{2026}{1-\pageref{LastPage}}{1/21; Revised 5/22}{9/22}{21-0000}{Octave Oliviers and Glenn Vinnicombe}

\ShortHeadings{Convergence of MC-O-PI with uniform updates}{Oliviers and Vinnicombe}
\firstpageno{1}

\begin{document}

%\title{Convergence of Monte Carlo Optimistic Policy Iteration with uniform updates}
\title{Convergence of Monte Carlo Optimistic Policy Iteration: Beyond Uniform State-Action Updates}

\author{\name Octave Oliviers \email ofao2@cam.ac.uk \\
       \addr Department of Engineering\\
       University of Cambridge\\
       Cambridge, UK
       \AND
       \name Glenn Vinnicombe \email gv103@cam.ac.uk \\
       \addr Department of Engineering\\
       University of Cambridge\\
       Cambridge, UK}

\editor{My editor}

\maketitle

\begin{abstract}%
The asymptotic behaviour of Monte Carlo optimistic policy iteration (MC-O-PI) is a long-standing open question. When the model of the environment is unknown, as is common in practice, the only known condition that guarantees convergence to optimality is impractical. In its canonical form, this condition requires that the episodes used for policy evaluation be initialised uniformly over the entire state-action space. This paper strictly relaxes that requirement. Specifically, we prove that initial-visit MC-O-PI converges to optimality even when updates are uniform only over the actions within each state. This allows episodes to start in different states at arbitrary frequencies; a realistic implementation when the state space is large or unknown but the action space in each state is manageable. The proof departs from the classical analysis of Tsitsiklis whose central commutativity argument no longer applies when states are updated at different frequencies. Instead, we first show that the mean-field dynamics of MC-O-PI generate monotonically improving policies when updates are uniform over the actions in each state, and then prove that noise cannot consistently prevent this improvement by extending the lock-in argument of the combined stability-ODE method. This approach suggests a new way to study optimistic policy-iteration algorithms in general.
\end{abstract}

\begin{keywords}
reinforcement learning, optimistic policy iteration, Monte Carlo control, convergence analysis, stability-ODE method
\end{keywords}

% \newpage
\section{Introduction}

% Many variants exist of the PI scheme
% on-off-policy
% used for planning (MCTS)
% bootstrapping
% optimism

% Many reinforcement-learning algorithms can be understood as instances of policy iteration, 
% alternating between estimating the action values of a policy and improving that policy using these estimates. This scheme is fundamental; it can learn an optimal policy in an unknown MDP using only sampled rewards and transitions. Optimistic reinforcement learning algorithms improve their policy before policy evaluation
% has converged. This makes them computationally attractive, but also difficult to analyse:
% the value estimate is stochastic and small
% changes near policy boundaries can drastically alter the subsequent dynamics.

Many reinforcement-learning algorithms follow the policy-iteration paradigm, which alternates between 
estimating the action values of a policy 
and improving the policy using those estimates. 
This framework is fundamental because it can learn an optimal policy for an unknown Markov decision process (MDP) using only sampled rewards and transitions.
Classical policy iteration improves the policy once the action-value estimates have converged.
In contrast, optimistic variants improve the policy after partial evaluation.
This reduces computation as it requires fewer evaluation steps between each improvement step.
However, it complicates the analysis because 
% the value estimates remain stochastic and 
small perturbations 
% near policy boundaries
to the value estimates
can drastically affect the policy and alter the subsequent dynamics.

Convergence guarantees for optimistic policy-iteration algorithms remain limited, even in the tabular setting.
Algorithms that use one-step bootstrapped estimates for evaluation, such as $Q$-learning, converge to optimality under standard asynchronous stochastic-approximation conditions.
When using multi-step bootstrapped estimates, convergence to optimality requires additional structure, such as combining several multi-step approximations~\citep{Harutyunyan2016QCorrections, Munos2016SafeLearning} or using a look-ahead mechanism~\citep{Winnicki2023OnEvaluation}.
Monte Carlo Optimistic Policy Iteration (MC-O-PI), which uses full-episode returns for evaluation, remains poorly understood. Despite its simplicity, the asymptotic behaviour of this foundational algorithm has been an open question for three decades~\citep{Sutton1999OpenLearning}.
Convergence to optimality is guaranteed when the MDP exhibits a feedforward structure \citep{Wang2022OnLearning, Lubars2021OptimisticStructure}
or when the discount factor is smaller than $1/2$ \citep{Delattre2025MarkovAlgorithm},
but, when the structure of the MDP is unknown, the only existing condition requires that all state-action pairs are updated at the same frequency \citep{Tsitsiklis2002OnIteration, Chen2018OnProblem, Liu2021OnStarts}.

% The requirement of uniform updating over all state-action pairs is
% often unrealistic in large problems,
% where the state distribution may be generated by the
% environment or by a search procedure. 
% In contrast, once a state has been selected, sampling
% uniformly among its available actions is usually straightforward.
% This paper proves that this weaker condition is in fact sufficient. That is, we show that initial-visit
% MC-O-PI converges almost surely to the optimal action values when updates are uniform
% only over the actions within each state. Different states may be updated with arbitrary
% frequencies. This strictly extends the previous convergence guarantees and makes MC-O-PI practical in settings where uniform state sampling is infeasible but uniform action sampling is easy, including when function approximation is being used over a large state-space.

Updating all pairs as frequently requires that episodes are initialised uniformly over the entire state-action space.
This condition is unrealistic in practice, especially when the state space is large.
However, once the initial state of an episode is selected, uniformly sampling the initial action is usually straightforward.
This paper proves that this weaker condition is sufficient.
Specifically, initial-visit MC-O-PI converges almost surely to the optimal action values when updates are uniform over the actions within each state.
Different states may thus be updated at arbitrary frequencies.
This result strictly extends the previous convergence guarantees 
and applies in settings where uniform state sampling is not feasible but uniform action sampling is easy.

The proof of this new theorem, stated as \autoref{thm: convergence mcopi uniform} below, requires new tools. 
Previous work demonstrates convergence by studying the evolution of the value estimates. The arguments rely centrally on a commutativity property of the Bellman operator that fails when updates are not uniform over all state-action pairs.
Instead, we track the evolution of the policy and find that the mean-field dynamics of MC-O-PI generate monotonically improving policies when updates are uniform over the actions within each state.
We then show that stochastic perturbations cannot systematically defeat this improvement because,
at any time step, a fixed positive constant lower bounds for the probability that ``if the current greedy policy is suboptimal, the next one is strictly better; if it is optimal, it remains optimal forever.''
As a result, almost surely, solutions eventually reach the optimal policy and remain there, causing the value estimate to converge to the corresponding optimal action values.

This paper is organised as follows. Section~2 introduces finite Markov decision
processes and formulates MC-O-PI as a stochastic difference inclusion. Section~3 proves the
main convergence theorem, first for the mean-field dynamics and then for the stochastic
algorithm. Section~4 discusses the practical implications of action-wise uniform updates
and possible extensions to related optimistic methods. 
% Section~5 concludes.
Unless stated otherwise, operations and relations between functions are elementwise. For instance, given functions $f, g : \mathcal X \to \mathcal Y$ we write
\begin{align*}
&(f + g)(x) = f(x) + g(x), \qquad \forall \, x \in \xset ,
    \\
&f \le g \iff f(x) \le g(x), \qquad \forall \, x \in \xset .
\end{align*}
Moreover, we compute norms as $\|f\| := \sup_{x \in \xset} |f(x)|$ 
and adopt the convention $\inf \varnothing = \infty$.
\section{Preliminaries}
\label{sec: mcopi}

% Monte Carlo optimistic policy iteration (MC-O-PI) is a reinforcement learning algorithm that can learn the optimal policy of a finite Markov decision process (MDP) without knowing the underlying parameters of that MDP.
% This section briefly introduces MDPs, then presents the stochastic difference inclusion that underlies MC-O-PI.

\subsection{Finite Markov decision processes}

% Reinforcement learning studies the interaction between an agent and an environment.
% ...
% this can be modelled using finite MDP
% ...
% At each time step \(t\), the agent observes a state \(s_t \in \mathcal S\) and selects an action \(a_t \in \mathcal A(s_t)\) according to a policy \(\pi(a_t | s_t)\).
% The environment then produces a reward \(r_t\) and a new state \(s_{t+1}\) according to its transition dynamics $p(s_{t+1}, r_t | s_t, a_t)$.
% % 
% The action value of a policy \(\pi\) for any pair $(s,a)$ is the expected discounted return obtained by taking action \(a\) in state \(s\), and then following policy \(\pi\):
% \[
% q_\pi(s,a)
% =
% \mathbb E \left[
% \sum_{t=0}^{\infty} \gamma^t r_t
% \;\middle|\;
% s_0 = s,\
% a_0 = a,\
% a_t \sim \pi(\cdot \mid s_t),\
% (s_{t+1},r_t) \sim p(\cdot,\cdot|s_t,a_t), \
% \forall\, t\ge 1
% \right].
% \]
% Under standard finite-MDP assumptions,
% there exists an optimal policy $\pi_*$ that satisfies $q_{\pi_*}(s,a) = \max_\pi q_\pi(s,a)$ for all $(s,a)$.
% Moreover, $\pi_*$ can be chosen greedily with respect to $q_{\pi_*}$, meaning
% % $q_{\pi_*}(s,\pi_*(s)) = \max_a q_{\pi_*}(s,a)$ for all $s$.
% $\pi_*(s) \in \arg\max_a q_{\pi_*}(s,a)$ for all $s$.

A Markov decision process evolves over discrete steps.
% with finite state space $\sset$ and finite action space $\aset(s)$ for each state $s \in \sset$.
At time \(t\), the agent observes a state \(s_t \in \mathcal S\) and selects an action \(a_t \in \mathcal A(s_t)\) according to a policy \(\pi(a_t | s_t)\).
The environment then produces a reward \(r_t\) and transitions to a new state \(s_{t+1}\) according to its dynamics $p(s_{t+1}, r_t | s_t, a_t)$.
% Taking action $a$ in state $s$ yields a finite reward $r(s,a) \in \real$ 
% and transitions to a new state $s'$ with probability $p(s' \mid s, a)$.
% The discount factor is $\discount \in [0,1]$.
If a policy $\pol$ is deterministic in state $s$, then $\pol(s)$ directly refers to the action $a$ where $\pol(a | s) = 1$.
We denote the set of deterministic policies by $\polset$.

For any policy $\pi$,
the state-value function $v_\pol : \sset \to \real$ 
measures the expected discounted return when following policy $\pol$ starting in state $s$:
\begin{align*} 
    v_\pol(s)
    :=
    % \E_{p} \big[ r(s_0, a_0) + \discount r(s_1, \pol)
    % \\
    % &\hspace{4.5cm} + \discount^2 r(s_2, \pol) + \dots \mid s_0 = s, a_0 = a \big] .
    \mathbb E_{p, \pi} \left[ \sum_{t=0}^{\infty} \gamma^t \, r_t  \ \bigg| \ s_0 = s \right] .
\end{align*}
The action-value function $q_\pol : \sset \times \aset \to \real$ measures the expected discounted return when taking action $a$ in state $s$ and then following policy $\pol$:
\begin{align*} 
    q_\pol(s,a)
    :=
    % \E_{p} \big[ r(s_0, a_0) + \discount r(s_1, \pol)
    % \\
    % &\hspace{4.5cm} + \discount^2 r(s_2, \pol) + \dots \mid s_0 = s, a_0 = a \big] .
    \mathbb E_{p, \pi} \left[ \sum_{t=0}^{\infty} \gamma^t \, r_t  \ \bigg| \ s_0 = s, a_0 = a \right] .
\end{align*}

A policy $\pi$ is at least as good as another policy $\pi'$ if $v_{\pi} \ge v_{\pi'}$ or, equivalently, $q_{\pi} \ge q_{\pi'}$.
We denote this relation by $\pi \succeq \pi'$.
If, additionally, there exists a state where $v_{\pi}(s) > v_{\pi'}(s)$, then $\pi$ is better than $\pi'$, denoted $\pi \succ \pi'$.

% For a policy \(\pi\), the action-value function $q_\pi$ assigns to each pair $(s,a)$ the expected discounted return obtained by taking action \(a\) in state \(s\), and then following policy \(\pi\):
% \[
% q_\pi(s,a)
% =
% \mathbb E_{\pi} \left[
% \sum_{t=0}^{\infty} \gamma^t r_t
% \;\middle|\;
% s_0 = s,\
% a_0 = a
% % a_t \sim \pi(\cdot \mid s_t),\
% % (s_{t+1},r_t) \sim p(\cdot,\cdot|s_t,a_t), \
% % \forall\, t\ge 1
% \right].
% \]

% Under standard finite-MDP assumptions,
% there exists an optimal policy $\pi_*$ that satisfies $q_{\pi_*}(s,a) = \max_\pi q_\pi(s,a)$ for all $(s,a)$.
% Moreover, $\pi_*$ may be chosen to be deterministic and greedy with respect to $q_{\pi_*}$, meaning
% $\pi_*(s) \in \arg\max_{a \in \mathcal A(s)} q_{\pi_*}(s,a)$ for all $s$.

This paper considers discounted or episodic MDPs,
meaning $\gamma \in [0,1)$
or episodes simulated with any stationary policy reach a terminal state almost surely.
The Policy Improvement Theorem then allows us to compare policies using only the action-value function of one of them~\citep{Howard1960DynamicProcesses}.
We rely on a strict version of that theorem.

% \label{thm: policy improvement theorem}\begin{theorem}
% [Policy Improvement Theorem]
% If policies $\pi$ and $\pi'$ satisfy
% \begin{align}
% % \label{eq: strict improvement theorem}
% q_\pi(s,\pi(s)) \ge q_\pi(s,\pi'(s))
% \end{align}
% for every state $s \in \sset$,
% then policy $\pi$ is at least as good as policy $\pi'$.
% If the inequality is strict in at least one state, 
% then policy $\pi$ better than policy $\pi'$.
% \end{theorem}

\begin{lemma}
\label{thm: strict PIT}
Given policies $\pi$ and $\pi'$, suppose that
\begin{align}
\label{eq: PIT condition}
q_\pi(s,\pi'(s)) \ge q_\pi(s,\pi(s))
% \qquad \forall s\in\mathcal S .
\end{align}
for every state $s\in\mathcal S$.
Then \(v_{\pi'} \ge v_\pi\). Moreover, if \(q_{\pi'}\neq q_\pi\), then the improvement is strict, meaning
there exists a state \(s\in\mathcal S\) such that
\[
v_{\pi'}(s)>v_\pi(s).
\]
\end{lemma}
\begin{proof}
By the Policy Improvement Theorem, inequality \eqref{eq: PIT condition} implies that $v_{\pi'} \ge v_\pi$.
Suppose that $v_{\pi'} = v_{\pi}$. Then for every pair $(s,a)$,
\begin{align*}
q_{\pi'}(s,a)
% = r(s,a) + \gamma \sum_{s' \in \mathcal S} p (s' \mid s,a) v_{\pi'}(s')
= \sum_{s'} \sum_{r} p (s', r \mid s,a) \big( r + \gamma v_{\pi'}(s') \big)
% = r(s,a) + \gamma \sum_{s' \in \mathcal S} p (s' \mid s,a) v_{\pi}(s')
= \sum_{s'} \sum_{r} p (s', r \mid s,a) \big( r + \gamma v_{\pi}(s') \big)
= q_{\pi}(s,a) .
\end{align*}
Thus, if $q_{\pi'} \neq q_\pi$, there exists at least one state where $v_{\pi'}(s) > v_{\pi}(s)$, and thus $\pi' \succ \pi$.
\end{proof}

Classical results guarantee that, for discounted or episodic MDPs, 
there exists an optimal policy $\pi_*$ that satisfies $q_{\pi_*}(s,a) = \max_\pi q_\pi(s,a)$ for all pairs $(s,a)$.
Moreover, $\pi_*$ may be chosen to be deterministic and greedy with respect to $q_{\pi_*}$, meaning
$\pi_*(s) \in \arg\max_{a \in \mathcal A(s)} q_{\pi_*}(s,a)$ for all $s$.
The converse is also true: if a deterministic policy is greedy with respect to its own action values, then it must be optimal.
This essential property allows us to verify whether a policy is optimal by inspecting its action values, rather than comparing it with all other policies.

\subsection{Monte Carlo optimistic policy iteration}

In many real-world applications, the transition dynamics and reward functions of the underlying MDP are unknown. Learning the optimal policy therefore requires interacting with the system to estimate expected returns from accumulated experience.
One fundamental framework for this is \textit{policy iteration}.
This approach tracks an action-value estimate $q$ and a policy $\pol$ and refines them by alternately evaluating the current policy and improving it.
This cycle is intended to converge to a policy that is greedy with respect to its own action values, and therefore optimal.

Monte Carlo optimistic policy iteration implements the policy-iteration framework by using 
\textit{Monte Carlo policy evaluation}, which simulates complete episodes of the MDP to update $q$ towards the action values of $\pol$,
and \textit{optimistic policy improvement}, which updates $\pol$ to obtain higher returns as anticipated by $q$.

Specifically, consider the function space $\qset \coloneqq \{ \mathcal S \times \mathcal A \to \mathbb R \}$.
% and let $k$ be an iteration index.
MC-O-PI generates a sequence $(q_k)_{k \ge 0}$ in $\qset$ as follows.
At the beginning of iteration $k$, the algorithm selects a deterministic policy $\pi_k$ that is greedy with respect to $q_k$:
\begin{align}
\label{eq: greedy policy definition}
    \pol_k(s) \in \arg\max_{a \in \aset(s)} q_k(s,a),
    \qquad \forall \, s \in \sset .
\end{align}
The algorithm then samples an initial state-action pair $(s_0, a_0)$ according to a sampling distribution $\start_k$,
simulates an episode of the MDP following policy $\pol_k$, and observes the discounted return for each state-action pair in the episode.
% Since the expected return for a given pair is $q_{\pi_k}(s,a)$, t
The observed return equals $q_{\pi_k}(s,a) + v_k(s,a)$, where $v_k(s,a)$ is a stochastic perturbation.
Finally, the estimate $q_k$ is updated at pairs appearing in the episode
according to
\begin{align}
\label{eq: MC eval update}
    q_{k+1}(s, a) 
    &= q_k(s,a) +  \lrate_k( q_{\pol_k}(s, a) + v_k(s, a) - q_k(s,a) ) ,
    % \\
    % &= q_k(s,a) + \lrate_k ( q_{\pol_k}(s, a) + v_k(s, a) - q_k(s,a)) .
\end{align}
where $(\lrate_k)_{k \ge 0}$ is a scalar learning rate sequence that satisfies the Robbins-Monro conditions:
\begin{align}
 \label{eq: robbins monro conditions}
    % 0 < \alpha_k \le 1 ,
    % \qquad
    \sum_{k =0}^\infty \alpha_k = \infty , 
    \qquad
    \sum_{k =0}^\infty \alpha_k^2 < \infty .
\end{align}
% $\sum_{k \ge 0} \lrate_k = \infty$ and $\sum_{k \ge 0} \lrate_k^2 < \infty$.

Let $\mathcal F_k$ represent the available information after $q_k$ is computed,
but before
selecting the greedy policy, 
sampling the initial state-action pair 
and simulating the episode.
% information available before simulating the $k^{\text{th}}$ episode and drawing the $k^{\text{th}}$ update mask, namely
% \mytodo{why only know $\pi_i$ up to k-1? F should contain $\pi_k$}
% \mytodo{showd $f_k$ and $\sigma_k$ also be in the filtration}
% \begin{align}
% \label{eq: filtration_def}
% % \mathcal F_k
% % \coloneqq
% % \sigma
% \Big\{
% q_0,
% \big(\alpha_i, \sigma_i, f_i, q_i,\pi_i,(s^{(i)}_t,a^{(i)}_t,r^{(i)}_{t})_{0\le t<T_i} \big)_{0\le i\le k-1}
% \Big\}.
% \end{align}

\subsubsection{Greedy policy}

Define the set-valued function $\grpol : \qset \rightrightarrows \polset$ that maps any action-value function $q$ to the set of deterministic policies that are greedy with respect to $q$:
\begin{align}
\label{eq: greedy pol}
    \grpol(q) \coloneqq \big\{ \pol \in \polset : q(s, \pol(s)) = \max_{a \in \aset(s)} q(s,a), \ \forall \, s \in \sset \big\} .
\end{align}
At every step, the policy $\pi_k$ thus satisfies 
\begin{align}
\label{eq: greedy policy def}
    \pi_k \in \grpol(q_k) .
\end{align}
If several policies are greedy with respect to $q_k$, we assume that ties are broken randomly and independently across states.
More precisely, we assume that
there exists a constant \(\beta>0\) such that
% for every time step \(k \ge 0\), every state \(s\in S\), and every action $a \in \arg\max_{A(s)} q_k(s,\cdot)$,
\begin{equation}
\label{eq:greedy-pol-tie-breaking}
\mathbb P(\pi_k(s)=a\mid \mathcal F_k)\ge \beta,
    \qquad 
    \forall \, k \ge 0, \
    \forall \, s \in \sset, \
    \forall \, a \in \arg\max_{\mathcal A(s)} q_k(s,\cdot) .
\end{equation}
% For example, under uniform tie-breaking,
% \[
% \mathbb P(\pi_k(s)=a\mid \mathcal F_k)
% =
% \frac{1}{|\arg\max_{a'\in A(s)} q_k(s,a')|}
% \ge \frac{1}{|A(s)|},
% \]
% so \eqref{eq:greedy-pol-tie-breaking} holds with
% \[
% \beta=\min_{s\in S}\frac{1}{|A(s)|}.
% \]
This assumption is mostly technical as ties rarely occur in practice, except if the algorithm is initialised in one.

\subsubsection{Update function}

Not all state-action pairs are necessarily updated after each episode.
Given an episode $\big( (s_t, a_t) \big)_{t \ge 0}$, there are three common update functions to decide which pairs to modify.
The \textit{initial-visit} method only updates the first state-action pair $(s_0, a_0)$.
The \textit{first-visit} method updates each state-action pair $(s_t, a_t)$ provided it has not appeared in the sequence $\big((s_{t'}, a_{t'})\big)_{t' < t}$.
The \textit{every-visit} method updates all state-action pairs of an episode using every observed return from each pair.

We only consider initial- and first-visit updates to continue
because every-visit updates can produce biased estimates of the action values~\cite[Theorem~7]{Singh1996ReinforcementTraces}.
In other words, under initial- or first-visit updates the stochastic perturbations $(v_k)_{k \ge 0}$ are unbiased.
Their variance is also bounded.
Since both properties hold uniformly over the policies, it follows that for every time $k \ge 0$ and every pair $(s,a)$ visited in the $k^\nth$ episode,
\begin{align}
&\mathbb E[v_k(s,a) \mid \mathcal F_k] = 0 ,
\\
&\mathbb E[v_k^2(s,a) \mid \mathcal F_k] \le c_v ,
\end{align}
for some constant $c_v \in [0,\infty)$ that depends on the parameters of the MDP, but is independent of the policy $\pi_k$.
% \begin{align*}
% \mathbb E[v_k(s,a) \mid \mathcal F_k, \pi_k] = 0 .
% \end{align*}
% Furthermore, there exists a constant $c_v \in [0,\infty)$ such that
% \begin{align*}
% \mathbb E[v_k^2(s,a) \mid \mathcal F_k, \pi_k] \le c_v .
% \end{align*}
% Since both properties hold uniformly over the policies, we even obtain
% \begin{align}
% &\mathbb E[v_k(s,a) \mid \mathcal F_k] = 0 ,
% \\
% &\mathbb E[v_k^2(s,a) \mid \mathcal F_k] \le c_v .
% \end{align}
Under initial- or first-visit updates, MC-O-PI thus equates to
\begin{align}
\label{eq: update q with chi}
    q_{k+1}(s,a) 
    =
    q_k(s,a) + \lrate_k u_k(s,a) (q_{\pol_k}(s, a) + v_k(s, a) - q_k(s,a))   
    % \begin{cases}
    %     \ (1-\lrate_k)q_k(s,a) + \lrate_k (q_{\pol_k}(s, a) + v_k(s, a)) &\text{if} \ u_k(s,a)=1
    %     \\
    %     \ q(s,a) &\text{if} \ u_k(s,a)=0
    % \end{cases}
\end{align}
where $u_k(s,a)=1$ if $(s, a)$ is to be updated, and 0 otherwise.
The indicator function $u_k$ depends on the initial state-action pair of the $k^{\text{th}}$ episode, the policy used for simulating it and the chosen update function.
Considering the properties of $v_k$,
if we fixed the policy $\pi_k = \pi$ for all $k$, solutions of \eqref{eq: update q with chi} would converge almost surely to $q_{\pi}$ as long as all state-action pairs are updated infinitely often~\citep[Proposition~4.1]{Bertsekas1996Neuro-DynamicProgramming}.

\subsubsection{Initial sampling distribution}

To ensure that all state-action pairs are updated sufficiently often,
we commonly impose \textit{exploring starts} on the sampling distribution of the initial state-action pair of each episode. This condition requires that every state-action pair has a non-zero probability of being selected as the start of an episode:
% In other words, for every index $k>0$, state $s \in \sset$ and action $a \in \aset(s)$, the sampling distribution $\start_k$ of the initial state-action pair satisfies
\begin{align*}
    \start_k(s,a) > 0,
    \qquad 
    \forall \, k \ge 0, \
    \forall \, s \in \sset, \
    \forall \, a \in \aset(s) .
\end{align*}
However,
we can effectively nullify this condition using quickly decaying sampling distributions.
For example, if we choose $\sigma_k$ such that $\start_k(s,a) \le 1/(k+2)^2$ for some pair $(s,a)$, there is a $50\%$ chance that no episode ever starts in $(s,a)$. Under initial-visit updates, that pair would thus never be updated. To prevent such issues, we assume a stronger condition called \textit{strict exploring starts}:
there exists a fixed $\sigma_{\min} \in (0,\infty)$ such that
\begin{align}
\label{eq:minimal-sampling}    
\sigma_k(s,a) \ge \sigma_{\min},
    \qquad 
    \forall \, k \ge 0, \
    \forall \, s \in \sset, \
    \forall \, a \in \aset(s) .
\end{align}

\subsubsection{Update distribution}

The probability of updating each state-action pair after an episode depends on the distribution of the initial state-action pair, the policy used for simulating the episode and the update function.
We gather these probabilities in the update distribution $\mu_k$ defined as
\[
\mu_k := \mathbb E [ u_k \mid \mathcal F_k, \sigma_k, \pi_k, f_k] .
\]
% $: \sset \times \aset \to [0,1]$ 
% such that 
% $u_k = \mu_k + v'_k$, where $v'_k$ is zero-mean noise.

% \mytodo{$\mu_k$ is actually independent of $\mathcal F_k$ given the other three}

This distribution controls the exploration-exploitation trade-off
as it determines how frequently each action is updated.
Exploration favours actions whose current estimated value is most uncertain, for instance by biasing $\mu_k$ towards actions that received few updates in the past.
Exploitation, on the other hand, prioritises updating actions with the highest estimated value. It thus biases $\mu_k$ towards actions that are greedy with respect to $q_k$.

Using the update distribution, 
we can combine the stochastic terms $u_k$ and $v_k$ in equation \eqref{eq: update q with chi} into a single term $w_k$ defined as
\begin{align}
\label{eq: definition combined noise}
    w_k
    \coloneqq 
    u_k \, v_k + (u_k - \mu_k) ( q_{\pol_k} - q_k ) .
\end{align}
This term incorporates the perturbations caused by the episode simulation and the asynchronous updates of state-action pairs.
It simplifies equation \eqref{eq: update q with chi} to
\begin{align}
\label{eq: mcopi not yet sri}
    q_{k+1}
    =
    q_k + \lrate_k \big( \update_k ( q_{\pol_k} - q_k ) + w_k \big) .
\end{align}

\subsubsection{Difference inclusion}

% \begin{definition}[MC-O-PI difference ]
After defining the sequences of learning rates $(\lrate_k)_{k \ge 0}$, start distributions $(\start_k)_{k \ge 0}$ and update functions $(f_k)_{k \ge 0}$,
we can combine equations \eqref{eq: greedy policy def} and \eqref{eq: mcopi not yet sri} into the difference inclusion
\begin{multline}
\label{eq: mcopi sa}
    q_\kpo
    \in
    \Big\{ q_k + \alpha_k \big( \mu_k(q_{\pol_k} - q_k) + w_k \big) 
    :
    \\
    \pol_k \in \grpol(q_k),
    \update_k = \update(\start_k, \pol_k, f_k),
    w_k \sim w(\start_k, \pol_k, f_k)
    \Big\} .
\end{multline}
% \end{definition}
We assume the following conditions throughout this paper.

\begin{assumption}[Standing assumptions]
\label{asp: standing}
\phantom{Assume}
\begin{enumerate}[label=(\arabic*)]
    % \item \textbf{Discounted or episodic MDP.} 
    % % The Bellman operators $T_\pi$ and $T$ are monotonic and are $\gamma$-contractions in the $L_\infty$-norm.
    % Either the discount factor satisfies $\gamma \in [0, 1)$ or each episode of the MDP reaches a terminal state almost surely.
    
    \item \textbf{Robbins--Monro learning rates.} The learning rates satisfy condition \eqref{eq: robbins monro conditions}.
    % The learning rates $(\lrate_k)$ satisfy
    % \begin{align}
    % \label{eq: robbins monro conditions}
    %     % 0 < \alpha_k \le 1 ,
    %     % \qquad
    %     \sum_{k =0}^\infty \lrate_\kpz = \infty , 
    %     \qquad
    %     \sum_{k =0}^\infty \lrate_\kpz^2 < \infty .
    % \end{align}
    Since these conditions require that $\alpha_k \to 0$, we assume, without loss of generality, that $0 < \alpha_k \le 1$ for all $k \ge 0$.

    \item \textbf{Exploring policies.}
    The greedy policies satisfy condition \eqref{eq:greedy-pol-tie-breaking}.

    \item \textbf{Strict exploring starts.}
    The sampling distributions satisfy condition \eqref{eq:minimal-sampling}. 
    % \begin{align}
    % \label{eq: strict exploring starts}
    %     \sigma_k(s,a) \ge \sigma_{\min}
    %     \qquad 
    %     \forall \, k \ge 0, \
    %     \forall \, s \in \sset, \
    %     \forall \, a \in \aset(s) ,
    % \end{align}
    % for some constant $\sigma_{\min} \in (0,\infty)$.
    
    % \item \textbf{Unbiased episode return.} The update functions are initial- of first-visit updates.
\end{enumerate}
\end{assumption}

% where the set-valued function $\grpol : \qset \rightrightarrows \polset$ maps an action-value function $q$ to the set of deterministic policies that are greedy with respect to $q$:
% \begin{align}
% \label{eq: greedy pol}
%     \grpol(q) \coloneqq \big\{ \pol \in \polset : q(s, \pol(s)) = \max_{a \in \aset(s)} q(s,a), \ \forall s \in \sset \big\} .
% \end{align}

System \eqref{eq: mcopi sa} is a switched linear system whose active dynamics depend on which policy is greedy with respect to the current estimate $q_k$.
Thus, as long as the greedy policy does not change, the dynamics of \eqref{eq: mcopi sa} remain the same.
% which are essentially a stochastic approximation of the differential equation
% \[
% \dot q = \update_k ( q_{\pol_k} - q ) .
% \]
% 
We can delimit the subspaces of $\qset$ with constant dynamics using the set-valued function $\graval : \polset \rightrightarrows \qset$ defined as
\begin{align}
\label{eq: greedy action values}
    \graval(\pol) \coloneqq \big\{ q \in \qset : q(s, \pol(s)) = \max_{a \in \aset(s)} q(s,a), \ \forall \, s \in \sset \big\} .
\end{align}
This function is the inverse of $P$, defined in \eqref{eq: greedy pol}, as it maps any deterministic policy $\pol$ to the set of action-value functions for which $\pol$ is greedy.
Since $\graval(\pol)$ forms a convex subset of $\qset$, we refer to $\graval(\pol)$ as the \textit{region} of policy $\pol$ in $\qset$.

% For instance, consider the MDP in \autoref{fig: description region smallest MDP} with $\sset = \{s\}$ and $\aset(s) = \{a,a'\}$.
% Any point $q \in \qset$ can be represented by a point in $\real^2$ with coordinates $( q(s,a), q(s,a') )$.
% As shown in \autoref{fig: show regions}, the line $q(s,a) = q(s,a')$ bisects $\qset$.
% The region $\graval(\pol)$ is the half-space where $q(s,a) \ge q(s,a')$, while $\graval(\pol')$ is the half-space where $q(s,a') \ge q(s,a)$.

% Suppose the update distributions are uniform over all state-action pairs, and hence equivalent to a scalar in \eqref{eq: mcopi sa}.
% Following the ODE-interpretation,
% MC-O-PI solutions then approximate $\dot q = q_{\pi} - q$ over $\graval(\pi)$ and $\dot q = q_{\pi'} - q$ over $\graval(\pi')$, or in short
% \begin{align}
% \label{eq: switched ode for interpretation}
% \dot q \in \{ q_\pi - q : \pi \in \grpol(q) \}.
% \end{align}
% As the learning rate decreases, MC-O-PI follows the streamlines in \autoref{fig: explain regions with streamlines} more closely (arbitrary position of $q_\pi$ and $q_{\pi'}$).
% On the boundary, either systems are possible since the boundary belongs to both regions. However, the tie-breaking must satisfy \eqref{eq:greedy-pol-tie-breaking}.

% \input{figures/single-state-flow-explanation}

%%%%%%%%%%%%%%%%%%%%%%%%%%%%%%%%%%%%%%%%%%%%%%%%%%%%%%%%%%%%%%%%%%%%%%%%%%
%%%%%%%%%%%%%%%%%%%%%%%%%%%%%%%%%%%%%%%%%%%%%%%%%%%%%%%%%%%%%%%%%%%%%%%%%%
\subsubsection{Basic properties}

The stochastic perturbations defined in \eqref{eq: definition combined noise} form a martingale difference sequence with bounded second moment.

\begin{lemma}
% [Unbiased perturbations with bounded conditional variance]
\label{thm: properties of noise}
% Under initial- or first-visit updates, 
There exists a constant $c_w \in [0, \infty)$ 
% independent of $\mu_k$
such that for every time $k \ge 0$, every state $s \in \sset$ and every action $a \in \aset(s)$,
\begin{align}
\label{eq: muw_mds}
% \E\!\left[\mu_k(s,a)\,w_k(s,a)\mid \mathcal F_k\right] &= 0,\\
&\mathbb E\!\left[w_k(s,a) \mid \mathcal F_k\right] = 0,
\\
\label{eq: muw_second_moment}
% \E\!\left[\mu_k^2(s,a)\,w_k^2(s,a)\mid \mathcal F_k\right]
&\mathbb E\!\left[w_k^2(s,a) \mid \mathcal F_k\right]
\le c_w(1+\|q_k\|^2) .
\end{align}
% In particular, $\mu_k(s,a)w_k(s,a)$ has bounded conditional variance uniformly in $k$ and $(s,a)$.
\end{lemma}

Using \autoref{thm: properties of noise} and the Robbins--Monro conditions,
we can apply standard martingale and stochastic approximation arguments to establish fundamental properties of MC-O-PI.

\begin{lemma}
% [Bounded solutions]
\label{thm: bounded limit sets}
Let $\polset_\infty$ denote the set of policies that appear infinitely often in a sequence $(\pi_k)_{k \ge 0}$ generated by MC-O-PI.
% \begin{align}
%     \polset_\infty \coloneqq \bigcap_{t = 0}^{\infty} \text{cl} \{\pi_k : k \geq t \} ,
% \end{align}
% Under \autoref{asp: standing},
The corresponding sequence $(q_k)_{k \ge 0}$ converges almost surely to the set
\begin{align*}
    \Big\{
    q \in \qset : 
    \min_{\pi \in \polset_\infty} q_\pi(s,a) \leq q(s,a) \leq \max_{\pi \in \polset_\infty} q_\pi(s,a), \forall s \in \sset, \forall a \in \aset
    \Big\} .
\end{align*}
Thus
there exists an almost surely finite constant $c_Q \in [0, \infty)$ such that for all $k \geq 0$
    \begin{align*}
        \| q_k \|_\infty^2 < c_Q .
    \end{align*}
\end{lemma}

If the target action values $q_{\pi_k}$ in \eqref{eq: mcopi sa} were fixed,
solutions would almost surely converge to this fixed target by Proposition~4.1 or Example~4.3 in~\citep{Bertsekas1996Neuro-DynamicProgramming}.
Thus, the asymptotic behaviour of the difference inclusion hinges on how these action values evolve over time.
To isolate these changes, we introduce \textit{transition times}.

% \noindent
% {\bf Theorem} {\it Let $u,v,w$ be discrete variables such that $v, w$ do
% not co-occur with $u$ (i.e., $u\neq0\;\Rightarrow \;v=w=0$ in a given
% dataset $\dataset$). Let $N_{v0},N_{w0}$ be the number of data points for
% which $v=0, w=0$ respectively, and let $I_{uv},I_{uw}$ be the
% respective empirical mutual information values based on the sample
% $\dataset$. Then
% \[
% 	N_{v0} \;>\; N_{w0}\;\;\Rightarrow\;\;I_{uv} \;\leq\;I_{uw}
% \]
% with equality only if $u$ is identically 0.} 
% \hfill\BlackBox

\begin{definition}
\label{def: transition times}
Given a sequence of policies $(\pi_k)_{k \ge 0}$,
define the transition times $(t_n)_{n \ge 0}$ as
\begin{align*}
t_0 &:= 0 ,
\\
t_{n+1} &:= \inf\{k>t_n : q_{\pi_k} \neq q_{\pi_{t_n}} \}.
\end{align*}
Thus, if $t_n = \infty$, then $t_{n+i} = \infty$ for all $i \ge 1$.
These times decompose the sequence $(\pi_k)_{k \ge 0}$ into blocks of constant action values, such that $q_{\pi_k} = q_{\pi_{t_n}}$ for all $k \in \{t_n, \dots, t_{n+1}-1 \}$.
\end{definition}

% \begin{lemma}
% Let \((\pi_k)_{k\ge 0}\) be a sequence of policies generated by MC-O-PI,
% and let \((t_n)_{n\ge 0}\) be the associated transition times.
% Under \autoref{asp: standing},
% if for some index $n \in \mathbb Z_{\ge 0}$, the policy $\pi_{t_n}$ is suboptimal, then
% \[
% \mathbb P( t_{n+1}<\infty \mid \mathcal F_{t_n} ) = 1.
% \]
% \end{lemma}
% \begin{proof}
% Let \((\pi_k)_{k\ge 0}\) be a sequence of policies generated by MC-O-PI,
% and let \((t_n)_{n\ge 0}\) be the associated transition times.

% If \(t_{n+1}=\infty\), then by \autoref{def: transition times},
% \(q_{\pi_k}=q_{\pi_{t_n}}\) for all \(k\ge t_n\).
% Hence the target remains fixed forever, so \(q_k\to q_{\pi_{t_n}}\) almost surely.

% Since \(\pi_{t_n}\) is suboptimal, it is not greedy with respect to \(q_{\pi_{t_n}}\). Thus there exist
% \(s\in S\) and \(a\in A(s)\) such that
% \[
% q_{\pi_{t_n}}(s,a) > q_{\pi_{t_n}}(s,\pi_{t_n}(s)).
% \]
% By convergence, the same strict inequality holds for all large \(k\), so
% \(\pi_k(s)\neq \pi_{t_n}(s)\) for all large \(k\) because \(\pi_k\) is greedy
% with respect to \(q_k\). Therefore \(q_{\pi_k}\neq q_{\pi_{t_n}}\) for some
% \(k>t_n\), contradicting \(t_{n+1}=\infty\). Hence \(t_{n+1}<\infty\) almost
% surely.
% \end{proof}

\begin{lemma}
\label{thm: suboptimal blocks finite}
If $\pi_{t_n}$ is not optimal, then $t_{n+1}<\infty$ almost surely.
\end{lemma}
\begin{proof}
Let $\pi := \pi_{t_n}$.
If $t_{n+1}=\infty$, then the target is fixed at $q_\pi$ forever and $q_k \to q_\pi$ almost surely~\cite[Proposition~4.1]{Bertsekas1996Neuro-DynamicProgramming}.
Since $\pi$ is suboptimal, it is not greedy with respect to $q_\pi$.
Hence there exists a pair $(s,a)$ such that $q_\pi(s,a) > q_\pi(s,\pi(s))$.
Consequently, as $q_k$ approaches $q_\pi$,
every greedy policy $\pi_k$ with respect to $q_k$ satisfies
$q_\pi(s,\pi_k(s)) > q_\pi(s,\pi(s))$
for all sufficiently large $k$.
The Policy Improvement Theorem indicates that $q_{\pi_k} \ne q_\pi$, contradicting $t_{n+1}=\infty$.
\end{proof}

% Clearly, if $\pi_{t_n}$ is suboptimal, then $t_{n+1}<\infty$.
% Otherwise the target would remain equal to \(q_{\pi_{t_n}}\) forever, forcing convergence to \(q_{\pi_{t_n}}\).
% However a suboptimal policy is not greedy with respect to its own action values.
% Thus no solution can get arbitrarily close to $q_{\pi_{t_n}}$ without causing the greedy policy to improve. Therefore no suboptimal block can persist indefinitely.
% 
\autoref{thm: suboptimal blocks finite} implies that if there exists a finite index $n$ such that $t_{n} = \infty$, the estimate $q_k$ converges to the action values of the last bounded transition time, which must be optimal.
Conversely, any suboptimal asymptotic behaviour requires an infinite sequence of bounded transition times.

\section{Results}
\label{sec: results}

% \textcolor{red}{remove references to asp 2}

Experiments show that initial-visit MC-O-PI can converge to suboptimal solutions under arbitrary update distributions, even when the Robbins–Monro 
and strict exploring-starts conditions hold.
% \citep[Example~5.12]{Bertsekas1996Neuro-DynamicProgramming}.
% \textcolor{red}{+ cite our neurips submission - GV I don't think we should, the chronology is this paper preceeds neurips (the story was decided, and most of it was written, before) }.
For instance, if at every step the actions associated with the current greedy policy are updated significantly less frequently than non-greedy actions, stable suboptimal equilibria can emerge even on the boundary of the optimal policy region.
% These equilibria effectively attract and stall MC-O-PI solutions.
This section shows that uniform update distributions prevent such issues because the mean-field dynamics then generate monotonically improving policies.
As a result, MC-O-PI converges to the optimal action values $\qopt$.

This result extends the guarantees of \citet{Tsitsiklis2002OnIteration}, \citet{Chen2018OnProblem} and \citet{Liu2021OnStarts} to a broader and more practical class of update distributions.
Specifically, while they prove convergence to optimality when updates are \textit{uniform over all state-action pairs}, we show that it remains true when updates are only \textit{uniform over all actions in each state}.

\begin{definition}[Uniform updates over all state-action pairs]
\label{asp: unbiased tsitsiklis}
The update distribution $\update$ is uniform over all state-action pairs
if for all states $s, s' \in \sset$ and actions $a \in \aset(s)$ and $a' \in \aset(s')$, we have
\begin{align*}
    \mu(s, a) = \mu(s', a') .
\end{align*}
\end{definition}

\begin{definition}[Uniform updates over actions in each state]
\label{asp: unbiased over policies}
The update distribution $\update$ is uniform over all actions in each state,
if for every state $s \in \sset$ and all actions $a, a' \in \aset(s)$
\begin{align*}
    \mu(s, a) = \mu(s, a') .
\end{align*}
\end{definition}

\autoref{asp: unbiased over policies} generalises \autoref{asp: unbiased tsitsiklis} because it allows updating different states at different frequencies.
For instance, under initial-visit updates, the initial state of each episode can now be sampled according to \textit{any} distribution provided the initial action in that state is sampled uniformly.
This additional flexibility disrupts Tsitsiklis' proof as the essential commutation step fails \citep[p.~64]{Tsitsiklis2002OnIteration}.
% However, we show that MC-O-PI still converges to the optimal solution.

This weaker condition is more practical for real-world applications, in particular when the state-space is large.
For instance, chess has $O(10^{45})$ different board layouts but only $O (10^1)$ legal moves per layout.
Uniformly sampling over the joint state-action space is thus not feasible, but sampling an action uniformly from a chosen state is straightforward.
Henceforth, \textit{uniform updates} refers to \autoref{asp: unbiased over policies} rather than \autoref{asp: unbiased tsitsiklis}.

We first demonstrate that the mean-field dynamics under uniform updates generate monotonically improving policies.
We then show that the stochastic perturbations in MC-O-PI cannot consistently prevent this improvement, and hence that MC-O-PI converges to the optimal action values.
Since the non-contracting, discontinuous and non-convex MC-O-PI dynamics preclude using classical methods, we derive convergence from first principles.

\subsection{Mean-field generates monotonically improving policies}

Given that the stochastic perturbations in initial- or first-visit MC-O-PI are unbiased (\autoref{thm: properties of noise}),
the mean-field update that arises when averaging out these perturbations is
\begin{align}
\label{eq: deterministic mcopi}
    q_{k+1}
    \in
    \big\{ q_k + \lrate_k \update_k ( q_{\pi_k} - q_k )
    :
    \pol_k \in \grpol(q_k),
    \update_k = \update(\start_k, \pol_k, f_k)
    \big\}
\end{align}
with learning rates $(\alpha_k)_{k\ge 0}$, start distributions $(\sigma_k)_{k\ge 0}$ and update functions $(f_k)_{k\ge 0}$.

\subsubsection{Policy improvement at every transition}

Using \autoref{thm: strict PIT},
we show that 
policies generated by the mean-field dynamics in \eqref{eq: deterministic mcopi} improve monotonically
because
\begin{align}
\label{eq: strict improvement theorem}
q_{\pi_{t_n}}(s,\pi_{t_{n+1}}(s)) \ge q_{\pi_{t_n}}(s,\pi_{t_n}(s))
\end{align}
for every state and at every transition.
% at every transition time,

\begin{proposition}
\label{thm: improving policies}
Let \((\pi_k)_{k\ge 0}\) be a sequence of policies generated by the mean-field dynamics in \eqref{eq: deterministic mcopi}
with uniform update distributions over all actions in each state (\autoref{asp: unbiased over policies}).
Let \((t_n)_{n\ge 0}\) be the associated transition times (\autoref{def: transition times}).
Under \autoref{asp: standing},
for every \(n \in \mathbb Z_{\ge 0}\) with \(t_{n+1}<\infty\), policy
\(\pi_{t_{n+1}}\) is better than \(\pi_{t_n}\).
\end{proposition}
\begin{proof}
Let $(q_k)_{k \ge 0}$ be a solution of \eqref{eq: deterministic mcopi} generated with update distributions $(\mu_k)_{k \ge 0}$, and let $(\pi_k)_{k \ge 0}$ be the resulting policy sequence with associated transition times $(t_n)_{n \ge 0}$.

Given policy $\pi_{k+1}$ is greedy with respect to $q_{k+1}$ for every time $k \ge 0$, we have $q_{k+1}(s,\pi_{k+1}(s)) \ge q_{k+1}(s, \pi_k(s))$ for every state $s \in \mathcal S$.
Using \eqref{eq: deterministic mcopi}, this becomes
\begin{multline*}
    q_k(s, \pol_{k+1}(s)) + \lrate_k \update_k(s, \pol_{k+1}(s)) \big( q_{\pi_k}(s, \pol_{k+1}(s)) - q_k(s, \pol_{k+1}(s)) \big)
    \\
    \ge
    q_k(s, \pol_k(s)) + \lrate_k \update_k(s, \pol_k(s)) \big( q_{\pi_k}(s, \pol_k(s)) - q_k(s, \pol_k(s)) \big) .
\end{multline*}
The strict exploring starts and the uniformity condition
imply that the update distributions are strictly positive and uniform over all actions in each state, meaning
$\update_k(s,\pi_k(s)) = \update_k(s,\pi_{k+1}(s)) =: \mu_k(s) > 0$.
Rearranging the previous inequality thus yields
\begin{align*}    
    q_{\pi_k}(s, \pol_{k+1}(s)) - q_{\pi_k}(s, \pol_k(s))
    \ge
    \dfrac{1 - \lrate_k \update_k(s)}{\lrate_k \update_k(s)} \big( q_k(s, \pol_k(s)) - q_k(s, \pol_{k+1}(s)) \big) .
\end{align*}
Since $\pi_k$ is greedy with respect to $q_k$ and, by assumption, $0 < \lrate_k \update_k(s) \le 1$, the right-hand side is nonnegative. As a result, we obtain
\begin{align}
\label{eq: ineq for improvement}
    q_{\pi_k}(s, \pi_{k+1}(s)) \ge q_{\pi_k}(s, \pi_k(s))
\end{align}
for every time $k \ge 0$ and every state $s \in \mathcal S$.

Now choose an $n$ such that $t_{n+1}<\infty$. Applying inequality \eqref{eq: ineq for improvement} at time $k=t_{n+1}-1$ and using the fact that $q_{\pi_{t_{n+1}-1}}=q_{\pi_{t_n}}$, we obtain
\begin{align*}
    q_{\pi_{t_n}}(s,\pi_{t_{n+1}}(s)) \ge q_{\pi_{t_{n}}}(s,\pi_{t_n}(s)) .
\end{align*}
By definition of the transition times, we also know that $q_{\pi_{t_n}} \neq q_{\pi_{t_{n+1}}}$.
As a result, \autoref{thm: strict PIT} implies that policy $\pi_{t_{n+1}}$ is better than $\pi_{t_n}$.
\end{proof}

\autoref{thm: improving policies} indicates that 
MC-O-PI with uniform updates would visit improving policies
if stochastic perturbations were averaged out, for instance by approximating $q_{\pi_k}$ as the average return from many episodes.

\subsubsection{Global convergence of mean-field solutions}

The monotone policy improvement directly yields global convergence to $\qopt$.

\begin{proposition}
\label{thm: convergence deterministic mcopi}
    Solutions of the mean-field dynamics in \eqref{eq: deterministic mcopi} converge to the optimal action values $\qopt$
    under \autoref{asp: standing},
    if the update distributions are uniform over all actions in each state (\autoref{asp: unbiased over policies}).
\end{proposition}

Technically, this result does not require all conditions in \autoref{asp: standing}. It suffices to impose Robbins--Monro learning rates and ensure that each state-action pair is updated infinitely often. We discuss this relaxation in section \ref{sec: relaxing conditions}.

\subsection{Initial-visit MC-O-PI converges with uniform updates}

% Propositions \ref{thm: improving policies}--\ref{thm: convergence deterministic mcopi} showed that mean-field solutions converge to the optimal action values through a sequence of improving policies when updates are uniform.

This section proves that stochastic perturbations in initial-visit MC-O-PI cannot consistently offset the policy improvement of mean-field solutions.
As a result, the algorithm converges to $\qopt$ when updates are uniform over the actions of each state.

%%%%%%%%%%%%%%%%%%%%%%%%%%%%%%%%%%%%%%%%%
\subsubsection{Limitations of standard convergence analyses}
\label{sec: limitations of existing methods}

While several standard frameworks can derive the convergence of a stochastic difference inclusion from the corresponding mean-field dynamics, they each face structural limitations when applied to MC-O-PI.

The robustness framework developed by \cite{Kellett2004SmoothInclusions}
indicates that the mean-field solutions in \eqref{eq: deterministic mcopi} are robust to perturbations \cite[Examples 2 \& 3]{Kellett2007SufficientInclusions}.
However, this method treats perturbations as an adversary choosing the worst possible value from the perturbed set.
Therefore, the perturbation radius must progressively decrease to zero as solutions approach an attractor.
Since the noise does not vanish in MC-O-PI, it cannot be contained within this shrinking radius.
% Put differently, by choosing to deal with worst-case perturbations, 
% this framework cannot exploit the martingale property and temporal averaging that naturally mitigate the cumulative impact of stochastic perturbations in MC-O-PI.

In contrast, the generalised ODE method exploits 
% the martingale property and 
temporal averaging that naturally mitigates the cumulative impact of noise in MC-O-PI \citep{Benaim2005StochasticInclusions}.
This method guarantees that solutions with uniform updates converge to an internally chain-transitive set of the differential inclusion
\begin{align}
\label{eq: limit differential inclusion}
    \dot q \in \textup{conv} \big\{ q_\pi - q : \pi \in P(q) \big\} ,
\end{align}
where $\text{conv}$ denotes the convex hull. 
However, the necessary convexification introduces spurious limit sets, such as suboptimal points that satisfy 
% the zero vector lies in the convex hull of update directions (
$0 \in \textup{conv} \{ q_\pi - q : \pi \in P(q) \}$,
or persistent cycles formed by Filippov solutions on the boundary between policy regions.
Thus, proving the global convergence of MC-O-PI with this method requires ruling out limit sets that are artifacts of the continuous-time approximation rather than genuine features of the algorithm.

The combined stability-ODE method relaxes the global convexity requirement of the generalised ODE-method to a compact subset \cite[Section 5.4]{Kushner2003StochasticApplications}.
If MC-O-PI visits that subset infinitely often, then, asymptotically, it follows mean-field solutions over that set.
However, the lock-in phenomenon at the core of this method requires
continuity of the limiting ODE over a fixed enlargement of the subset under consideration \citep{Borkar2002OnApproximation}.
Consequently, it is not sufficient that MC-O-PI solutions visit the region of a particular policy infinitely often. Instead, solutions must visit a fixed compact interior of that region infinitely often.
Since decreasing learning rates cause successive entries into a policy’s region to occur closer and closer to its boundary,
this method misses a way to establish that solutions cover the gap between the boundary and a fixed compact interior of a policy's region infinitely often.

% Why cannot use ODE method
% cannot use this framework of ODE method,
% or the generalisations by $\cite{Perkins2012AsynchronousInclusions, Ramaswamy2017AInclusions}$ because is not convex
% cannot use \cite{Liu2025TheNoise} because is not lipschitz
% Thus no guarantee solutions to resulting differnetial inclusion

The proof in the following sections addresses this limitation.
Namely, we first construct a mechanism that ensures solutions move away from boundaries despite starting closer and closer,
and then exploit the lock-in phenomenon to derive convergence to the optimal action values. 
This approach yields a fixed lower bound on the probability that policies improve at transitions.
The unlimited number of attempts then guarantee that MC-O-PI solutions eventually reach the optimal policy and converge to the associated action values.

% \textcolor{red}{continuity of other methods}
% continuity allows other frameworks \cite{Benaim2005StochasticInclusions, Borkar2000TheLearning} to use gronwall which is more easy
% however here discontinuity and non convexity make it quite a bit more complex
% so instead of being able to rely on the existence of the solutions fo the deterministic differential inclusion, we need to ...

% proof is essentially apply borkar lock in method but ...

%%%%%%%%%%%%%%%%%%%%%%%%%%%%%%%%%%%%%%%%%
\subsubsection{Lower bound on the probability of policy improvement at transitions}

Consider the sequences $(q_k)_{k\ge 0}$ and $(\pi_k)_{k\ge 0}$ generated by MC-O-PI, and let $(t_n)_{n \ge 0}$ be the associated transition times.
Over each interval $\{t_n, \dots, t_{n+1}\}$, we compare the evolution of $q_{k}$ at $(s,\pi_{t_n}(s))$ and pairs $(s,a)$ that would worsen the policy $\pi_{t_n}$, namely
\[
\mathcal B_{\pi_{t_n}} := \big\{ (s,a) : s \in \mathcal S, \ a \in \mathcal A(s), \ q_{\pi_{t_n}}(s,a) < q_{\pi_{t_n}}(s, \pi_{t_n}(s)) \big\} .
\]
We show that, with a fixed probability and for sufficiently large $n$,
for every $k \in \{t_n, \dots, t_{n+1}\}$
the greedy policy satisfies $(s, \pi_k(s)) \notin \mathcal B_{\pi_{t_n}}$.
As a result, if $t_{n+1}<\infty$ we obtain
\begin{align}
\label{eq: goal of proof stoch improvement}
q_{\pi_{t_n}} ( s, \pi_{t_{n+1}}(s) )
\ge
q_{\pi_{t_n}} ( s,\pi_{t_n}(s) )
\end{align}
and therefore policy $\pi_{t_{n+1}}$ is better than $\pi_{t_n}$ by \autoref{thm: strict PIT}.
Or if $t_{n+1} = \infty$,
\autoref{thm: suboptimal blocks finite} implies that $\pi_{t_n} = \pi_*$ and $q_k$ converges to the associated optimal action values.
In other words, the policy either improves, or it is already optimal and stays optimal.

First, let us impose one additional assumption on the learning rates.

\begin{assumption}[Comparability of learning rates]
\label{asp: additional conditions on learning rate}
    The learning rates $(\alpha_k)_{\ge 0}$ are nonincreasing,
    meaning $\alpha_{k+1}\le \alpha_k$,
    and there exist a time $K_\alpha<\infty$ and a constant $\rho_\alpha\in(0,1]$ such that 
    for every $k\ge K_\alpha$
    and every $i \in \{0, \dots, \left\lfloor 1 / \alpha_k \right\rfloor\}$
    \begin{equation}
    \label{eq:comparison}
    \alpha_{k+i}\ge \rho_\alpha \alpha_k .
    \end{equation}
\end{assumption}

\autoref{asp: additional conditions on learning rate} is a common regularity condition when analysing asynchronous stochastic approximations
\citep{Borkar1998AsynchronousApproximations, Perkins2012AsynchronousInclusions, Liu2025TheNoise}.
It ensures the learning rates do not decay too quickly as they remain ``comparable'' to $\alpha_k$ for at least $\lfloor 1/\alpha_k \rfloor$ successive steps.
Many classical learning-rate schemes satisfy this condition, such as $\alpha_k = c_1 / (k+c_2)^{c_3}$ with $c_1>0$, $c_2 \ge 0$ and $c_3 \in (1/2, 1]$.

\begin{proposition}
\label{thm: fixed prob that bad actions dont become greedy}
    Let $(\pi_k)_{k\ge 0}$ be a sequence of policies generated by MC-O-PI
    under \autoref{asp: standing},
    with initial-visit updates,
    learning rates that satisfy the comparability assumption (\autoref{asp: additional conditions on learning rate}),
    and update distributions that are uniform over all actions in each state (\autoref{asp: unbiased over policies}).
    Let $(t_n)_{n\ge 0}$ be the associated transition times.
    There exist a constant $p>0$ and an almost surely finite time $K$ such that
    for every $n \in \mathbb Z_{\ge 0}$ with $K \le t_n < \infty$,
    \begin{align}
    \label{eq: porb better action value}
    \mathbb P \big(
    \forall \, s \in \sset, \
    \forall \, k \in \{t_n, \dots , t_{n+1} \}
    \ : \
    (s, \pi_k(s)) \notin \mathcal B_{\pi_{t_n}}
    \mid \mathcal F_{t_n}
    \big) \ge p .
    \end{align}
\end{proposition}

We track the distance between $(s,\pi_{t_n}(s))$ and pairs in $\mathcal B_{\pi_{t_n}}$ using gap variables.

\begin{definition}[Gap variables]
\label{def: gap variables}
% Let $\{t_n\}_{n\ge 0}$ be the transition times associated with the policy sequence generated by system \eqref{eq: stochastic rep} under the setting of \autoref{thm: iv mcopi converges as}.
Fix a finite transition index $n \in \mathbb Z_{\ge 0}$
and write $\pi := \pi_{t_n}$.
% For each state $s \in \mathcal S$ and action $a \in \mathcal A(s)$, the gap is initialised at time $t_n$ as
For each $k \in \{t_n, \dots, t_{n+1}\}$, define
\begin{align*}
% \label{eq: def gap initialisation}
% d_{t_n}(s,a) = q_{t_n}(s,\pi(s)) - q_{t_n}(s,a) ,
d_{k}(s,a) = q_{k}(s,\pi(s)) - q_{k}(s,a) .
\end{align*}
While $k < t_{n+1}$, each gap evolves as
\begin{align}
\label{eq: def gap evolution}
d_{k+1}(s,a) = d_k(s,a) + \alpha_k \left( \sigma_k(s)(d_\pi(s,a) - d_k(s,a)) + \xi_k(s,a) \right)
\end{align}
where we simplified the uniform update distributions under initial-visit updates to
\begin{align*}
    \sigma_k(s) := \sigma_k(s,a) = \mu_k(s,a) = \mu_k(s,\pi(s)),
\end{align*}
the target gap $d_\pi$ is
\begin{align*}
% \label{eq: def policy gap}
    d_\pi(s,a) := q_{\pi}(s,\pi(s)) - q_{\pi}(s,a) ,
\end{align*}
and the gap noise $\xi_k$ is
\begin{align*}
% \label{eq:xi-def}
    \xi_k(s,a) &:=
w_{k}(s,\pi(s)) - w_{k}(s,a) .
% \\
% \nonumber
% &\: = u_k(s,\pi(s))v_k(s,\pi(s)) - u_k(s,a) v_k(s,a)
% \\
% \nonumber
% &\: \quad + \bigl( u_k(s,\pi(s)) - \mu_k(s) \bigr) \bigl( q_\pi(s,\pi(s)) - q_k(s,\pi(s)) \bigr)
% \\
% \nonumber
% &\: \quad - \bigl( u_k(s,a) - \mu_k(s) \bigr) \bigl( q_\pi(s,a) - q_k(s,a) \bigr) .
\end{align*}
\end{definition}

\begin{lemma}
\label{thm: properties noise gap}
There exist a constant $c_\xi \in [0,\infty)$ 
and an almost surely finite time $K_\xi$ such that
for every transition index $n \in \mathbb Z_{\ge 0}$ satisfying $t_n \ge K_\xi$,
\begin{align}
\label{eq:xi-mean}
&\mathbb E [ \xi_{k}(s,a) \mid \mathcal F_k ] = 0
\\
\label{eq:xi-var}
&\mathbb E [ \xi_{k}^2(s,a) \mid \mathcal F_k ] \le c_\xi
\end{align}
for every $k \in \{t_n, \dots, t_{n+1}-1\}$, every $s \in \mathcal S$, and every $a \in \mathcal A(s)$.
\end{lemma}
\begin{proof}
\autoref{thm: properties of noise}
% requires that $w_k(s,a)$ is unbiased.
directly proves \eqref{eq:xi-mean} since
\[
\mathbb E \!\big[ \xi_{k}(s,a) \mid \mathcal F_k \big] 
=
\mathbb E \!\big[ w_{k}(s,\pi_{t_n}(s)) \mid \mathcal F_k \big] - \mathbb E \!\big[ w_{k}(s,a) \mid \mathcal F_k \big] 
= 
0 .
\]
Further,
\autoref{thm: bounded limit sets} implies that
% $q_k$ converge almost surely to the set $\bigl\{ q \in \mathcal Q : q_{\worst{\pi}} \le q \le q_{\best{\pi}} \bigr\}$, where $\overline \pi$ and $\pi_*$ are the worst and best policies of the MDP, respectively.
% Thus, 
there exists an almost surely finite time $K_{\xi}$ such that $\| q_k \| \le c_{\mathcal P} + 1$ for every $k \ge K_{\xi}$,
with $c_{\mathcal P} := \max_{\pi \in \mathcal P}\| q_\pi \|_\infty$.
Substituting this into \eqref{eq: muw_second_moment} yields
\[
\mathbb E[w_k^2(s,a) \mid \mathcal F_k] 
\le c_w(1+\|q_k\|^2)
\le c_w(1+(c_{\mathcal P} + 1)^2)
\qquad
\forall \, k \ge K_{\xi}, \
\forall \, s \in \mathcal S, \
\forall \, a \in \mathcal A(s).
\]
Using $(x-y)^2 \le 2 ( x^2 + y^2)$ then proves
\eqref{eq:xi-var}
because for every $k \ge K_{\xi}$, every $s \in \mathcal S$, and all $a, a' \in \mathcal A(s)$
\begin{align*}   
\mathbb E\!\left[ \bigl(w_k(s,a) - w_k(s,a') \bigr)^2 \ \big| \ \mathcal F_k \right]
&\le
2 \left( \mathbb E\!\left[ w_k^2(s,a) \mid \mathcal F_k \right] 
+ \mathbb E\!\left[ w_k^2(s,a') \mid \mathcal F_k \right] \right)
\\
&\le
4 \, c_w \left( 1 + (c_{\mathcal P}+1)^2 \right)
=: c_\xi .
\end{align*}
\end{proof}

% \textcolor{red}{need $\mathcal B_{\pi_k} \equiv \mathcal B_{\pi}$ for all $k \in \{t_n, \dots t_{n+1}-1\}$? then easy to get proof, but if not necessary, dont add it}
Over each interval $\{t_n, \dots, t_{n+1}\}$,
we construct three events that guarantee that the gap $d_k$ of each pair in $\mathcal B_{\pi_{t_n}}$ is strictly positive up to time $t_{n+1}$,
and consequently that no such pair is greedy.
Specifically, we fix a margin $\delta>0$.
The \textit{seed} event $\mathcal E^s_{t_n}$ first forces each ``bad'' gap at least up to $O(\alpha_k)$.
The \textit{growth} event $\mathcal E^g_{t_n}$ then drives them from $O(\alpha_k)$ to $\delta$.
Finally, the \textit{lock-in} event $\mathcal E^l_{t_n}$ ensures that, after reaching $\delta$, they remains strictly positive up to time $t_{n+1}$.

Concretely, define the minimum target gap across all policies
\[
d_{\min}
:=
\min_{\pi\in\mathcal P}\ \inf_{(s,a)\in \mathcal B_\pi} d_\pi(s,a) > 0 .
\]
% with the convention that $\min\varnothing=\infty$.
Using the constants $\sigma_{\min}$ and $\rho_\alpha$ from \autoref{asp: standing}.(3) and \autoref{asp: additional conditions on learning rate}, respectively,
we choose $\delta$ such that
\begin{equation}
\label{eq:delta-choice}
0<\delta<\min\left\{\frac{d_{\min}}{8},\,\frac{\rho_\alpha \sigma_{\min} d_{\min}}{32}\right\}.
\end{equation}
This allows us to further define the constants
\begin{align}  
\label{eq:c-def}
% \label{eq:cs-def}
c_s := \left( \frac{d_{\min}}{4} - \delta \right) > 0,
% \\
% \label{eq:cg-def}
\qquad
c_g := \sigma_{\min} ( d_{\min} - 2 \delta) > 0,
% \\
% \label{eq:c-def}
\qquad
c := \min \left\{ c_s, c_g \right\} > 0.
\end{align}

%%%%%%%%%%%%%%%%%%%%%%%%%%%%%%%%%%%%%%%%%%%%%%%%%%%%%%%%%%%%%%%
\paragraph{Seed event pushes gap up to $\bm{O(\alpha_k)}$}

Let $M \in \mathbb N$ denote a finite time horizon whose value will be fixed later.
For any time $k \in \{t_n, \ldots, t_{n+1}-1\}$, we define the seed event differently if the value of the greedy action is under- or overestimated.
% Concretely, for any index $n \in \mathbb Z_{\ge 0}$,
% any time $k \in \{t_n, \dots, t_{n+1}-1\}$,
% any state $s \in \mathcal S$,
% and any action $a \in \mathcal A(s)$,
Namely, if $q_{\pi_{t_n}}(s,\pi_{t_n}(s)) - q_{k}(s,{\pi_{t_n}}(s)) \ge d_{\min}/2$,
we update the pair $(s,{\pi_{t_n}}(s))$ for $M$ consecutive steps and require that the noise at each step is not too negative. We gather these conditions in the event
\begin{align*}
    \mathcal E^+_k(s,a)
    := 
    \bigcap_{k'=k}^{k+M-1}
    \left\{
    u_{k'}(s,\pi_{t_n}(s)) = 1, \
    v_{k'}(s,\pi_{t_n}(s)) \ge -\frac{d_{\min}}{4}
    \right\} .
\end{align*}
On the other hand, if $q_{\pi_{t_n}}(s,\pi_{t_n}(s)) - q_{k}(s,{\pi_{t_n}}(s)) < d_{\min}/2$,
we update $(s,a)$ and require that the noise is not too positive:
\begin{align*}
    \mathcal E^-_k(s,a)
    := 
    \bigcap_{k'=k}^{k+M-1}
    \left\{
    u_{k'}(s,a) = 1, \
    v_{k'}(s,a) \le \frac{d_{\min}}{4}
    \right\} .
\end{align*}
Altogether, we obtain the seed event
\begin{multline}
\label{eq: def good prefix}
    \mathcal E^s_k(s,a) := 
    \left( \left\{ q_{\pi_{t_n}}(s,\pi_{t_n}(s)) - q_{k}(s,{\pi_{t_n}}(s)) \ge \frac{d_{\min}}{2} \right\} \cap \mathcal E^+_k(s,a) \right) 
    \\
    \bigcup 
    \left( \left\{ q_{\pi_{t_n}}(s,\pi_{t_n}(s)) - q_{k}(s,\pi_{t_n}(s)) < \frac{d_{\min}}{2} \right\} \cap \mathcal E^-_k(s,a) \right) .
\end{multline}
We also define the hitting time
\[
\tau^s_k(s,a) := \inf \big\{ k' \in \{k+1, \dots, k+M \} : q_{k'}(s,\pi_{t_n}(s))-q_{k'}(s,a) \ge \delta \} ,
\]
and the accumulated learning rate mass between times $k$ and $k'>k$ as
\[
A(k,k') := \sum_{i=k}^{k'-1} \alpha_{i} .
\]
% with the convention that $A(k, k) = 0$.
We now show that, on $\mathcal E^s_k(s,a)$, ``bad'' gaps grow at least up to $O(\alpha_k)$ in $M$ steps.

\begin{lemma}[Seed event]
\label{thm: event to reach a_k margin}
Consider the setting of \autoref{thm: fixed prob that bad actions dont become greedy}.
For every index $n \in \mathbb Z_{\ge 0}$,
every pair $(s,a) \in \mathcal B_{\pi_{t_n}}$
and every time $k \in \{t_n, \dots, t_{n+1}-1\}$,
the joint event
\begin{align}
\label{eq: joint seed event}
\mathcal E^s_k(s,a)
\cap
% \big\{0 \leq q_{k}(s,\pi_{t_n}(s)) - q_{k}(s,a) < \delta \big\}
\big\{0 \leq d_k(s,a) < \delta \big\}
\end{align}
guarantees that 
% for every $k'$ satisfying 
% % $k < k' \le \big( \tau^s_k(s,a) \wedge (k+M) \wedge t_{n+1} \big)$,
% $k < k' \le \min\{ \tau^s_k(s,a), k+M, t_{n+1} \}$,
\begin{align}
\label{eq: lower bound during good prefix}
    d_{k'}(s,a)
    % q_{k'}(s,\pi_{t_n}(s)) - q_{k'}(s,a)
    > c_s A(k,k') > 0
    \qquad 
    \forall \, k < k' \le \min\{ \tau^s_k(s,a), k+M, t_{n+1} \} .
\end{align}
% for every $k'$ that satisfies $k < k' \le \big( \tau^s_k(s,a) \wedge (k+M) \wedge t_{n+1} \big)$.
Moreover, 
% the probability of $\mathcal E^s_k(s,a)$ is lower bounded by
\begin{align}
\label{eq: probability reach a_k margin}
    \mathbb P ( \mathcal E^s_k(s,a) \mid \mathcal F_k, \pi_{t_n} ) \ge \left( \sigma_{\min} \frac{d_{\min}^2}{d_{\min}^2 + 16 c_v} \right)^M > 0 .
\end{align}
\end{lemma}

\begin{proof}
Fix a transition index $n \in \mathbb Z_{\ge 0}$ and let $\pi := \pi_{t_n}$.
Fix a pair $(s,a) \in \mathcal B_\pi$,
define the sequence $\big( d_k(s,a) \big)_{k=t_n}^{t_{n+1}}$ as in \autoref{def: gap variables}, and write $d_k := d_k(s,a)$ and $d_\pi := d_\pi(s,a)$ for brevity.
Fix a time $k \in \{t_n, \dots, t_{n+1}-1\}$ and write $\tau^s_k := \tau^s_k(s,a)$.

% We split the proof along the two mutually exclusive cases that make up $\mathcal E^s_k(s,a)$.

When $q_\pi(s,\pi(s)) - q_{k}(s,\pi(s)) \ge d_{\min}/2$,
% On 
event $\mathcal E^+_k(s,a)$
% only $(s, \pi(s))$ is updated.
% during $M$ consecutive steps,
% and $v_{k'}(s,\pi(s)) \ge -d_{\min}/4$ for every $k' \in \{k, \dots, k+M-1\}$.
implies that for every $k'$ satisfying
% $k \le k' < \big((k+M) \wedge t_{n+1} \big)$,
$k \le k' < \min\{ k+M, t_{n+1} \}$,
\begin{align}
\nonumber
    d_{k'+1}
    &= d_{k'} + \alpha_{k'} \big( q_\pi(s,\pi(s)) - q_{k'}(s,\pi(s)) + v_{k'}(s,\pi(s)) \big)
    \\
\label{eq: plug noise into update of seed stage case 1}
    &\ge d_{k'} + \alpha_{k'} \left( q_\pi(s,\pi(s)) - q_{k'}(s,\pi(s)) - \frac{d_{\min}}{4} \right) .
\end{align}
Given that $0 \le d_{k} < \delta$ and $d_{k'} < \delta$ while $k'< \tau^s_k$,
it follows that 
\begin{align}
% \nonumber
\label{eq: lower bound update step in seed stage case 1}
    q_\pi(s,\pi(s)) - q_{k'}(s,\pi(s))
    = q_\pi(s,\pi(s)) - d_{k'} + d_{k} - q_{k}(s,\pi(s))
    % \\
    > \frac{d_{\min}}{2} - \delta > 0
\end{align}
as long as
% $k \le k' < \big( \tau^s_k \wedge (k+M) \wedge t_{n+1} \big)$,
$k \le k' < \min\{ \tau^s_k , k+M, t_{n+1} \}$.
Equations \eqref{eq: plug noise into update of seed stage case 1} and \eqref{eq: lower bound update step in seed stage case 1}
thus indicate that, on \eqref{eq: joint seed event},
% for every $k'$ satisfying
% $k \le k' < \min\{ \tau^s_k , k+M, t_{n+1} \}$,
% the gap strictly increases at each time $k'$ in the restricted range:
\begin{align}
\label{eq: final lower bound case 1}
    d_{k'+1}
    > d_{k'} + \alpha_{k'} \left( \frac{d_{\min}}{4} - \delta \right)
    = d_{k'} + \alpha_{k'} c_s
    % > 0 .
    \qquad
    \forall \, k \le k' < \min\{ \tau^s_k , k+M, t_{n+1} \} .
\end{align}

Similarly, when $q_\pi(s,\pi(s)) - q_{k}(s,\pi(s)) < d_{\min}/2$
% event $\mathcal E^-_k(s,a)$ guarantees that
% only $(s, a)$ is updated during $M$ consecutive steps, and $v_{k'}(s,a) \le d_{\min}/4$ for every $k' \in \{k, \dots, k+M-1\}$.
we obtain
\begin{align}
% \nonumber
    d_{k'+1}
    % &= d_{k'} + \alpha_{k'}(q_{k'}(s,a) - q_\pi(s,a) - v_{k'}(s,a))
    % \\
\label{eq: plug noise into update of seed stage case 2}
    &\ge d_{k'} + \alpha_{k'} \left( q_{k'}(s,a) - q_\pi(s,a) - \frac{d_{\min}}{4} \right)
\end{align}
as long as
% $k \le k' < \big((k+M) \wedge t_{n+1} \big)$,
$k \le k' < \min\{ k+M, t_{n+1}\}$.
Since $q_{k}(s,\pi(s))$ is unchanged on $\mathcal E^-_k(s,a)$
and $d_{k'} < \delta$ while $k' < \tau^s_k$,
it follows that
\begin{align}
% \nonumber
    q_{k'}(s,a) - q_\pi(s,a)
    = d_\pi - \big( q_\pi(s,\pi(s)) - q_{k}(s,\pi(s)) \big) - d_{k'}
    % \\
\label{eq: lower bound update step in seed stage case 2}
    % > d_{\min} - \frac{d_{\min}}{2} - \delta
    % \ge \frac{d_{\min}}{2} - \delta 
    > \frac{d_{\min}}{2} - \delta 
\end{align}
as long as
% $k \le k' < \big( \tau^s_k \wedge (k+M) \wedge t_{n+1} \big)$,
$k \le k' < \min\{ \tau^s_k, k+M, t_{n+1} \}$.
Hence \eqref{eq: plug noise into update of seed stage case 2} and \eqref{eq: lower bound update step in seed stage case 2} confirm that, on \eqref{eq: joint seed event},
% the gap strictly increases at each time $k'$ in the restricted range:
\begin{align}
\label{eq: final lower bound case 2}
    d_{k'+1}
    > d_{k'} + \alpha_{k'} \left( \frac{d_{\min}}{4} - \delta \right)
    = d_{k'} + \alpha_{k'} c_s
    % > 0 .
    \qquad
    \forall \, k \le k' < \min\{ \tau^s_k , k+M, t_{n+1} \} .
\end{align}
By induction on \eqref{eq: final lower bound case 1} or \eqref{eq: final lower bound case 2},
% and shifting the index by 1 to include time step $k'+1$,
we obtain
\begin{align}
\label{eq: increase in d^b case B}
    d_{k'}
    > d_{k} + c_s A(k,k')
    \ge c_s A(k,k') > 0
    \qquad
    \forall \, k < k' \le \min\{ \tau^s_k , k+M, t_{n+1} \} .
\end{align}
% for every $k'$ satisfying
% $k < k' \le \min\{ \tau^s_k , k+M, t_{n+1} \}$,
% as long as $k < k' \le \big( \tau^s_k \wedge (k+M) \wedge t_{n+1} \big)$, 
which proves \eqref{eq: lower bound during good prefix}.

Finally, 
% we lower-bound the probability of $\mathcal E^s_k(s,a)$.
given that $v_k$ is unbiased and its variance is bounded by $c_v$ uniformly over the policies,
Cantelli's inequality implies that for every $k' \ge 0$, every $s \in \mathcal S$ and every $a \in \mathcal A(s)$,
\begin{align}
&\mathbb P \! \left( v_{k'}(s,a) \ge -\frac{d_{\min}}{4} \,\middle|\, \mathcal F_{k'}, u_{k'}(s,a)=1 \right) \ge \frac{d_{\min}^2}{d_{\min}^2 + 16 c_v},
\label{eq:noise-lower}
\\
&\mathbb P \! \left( v_{k'}(s,a) \le \frac{d_{\min}}{4} \,\middle|\, \mathcal F_{k'}, u_{k'}(s,a)=1 \right) \ge \frac{d_{\min}^2}{d_{\min}^2 + 16 c_v}.
\label{eq:noise-upper}
\end{align}
Since the case choice in \eqref{eq: def good prefix} is measurable on $\{\mathcal F_{k}, \pi_{t_n}\}$ and the minimum sampling probability of any state-action pair is $\sigma_{\min}$ by \autoref{asp: standing}.(3),
% under strict exploring starts, 
repeated conditioning over $M$ steps yields
\begin{align}
\label{eq: prob of good prefix}
\mathbb P \big( \mathcal E^s_k(s,a) \mid \mathcal F_{k} , \pi_{t_n} \big) \ge \left( \sigma_{\min} \frac{d_{\min}^2}{d_{\min}^2 + 16 c_v} \right)^M > 0 .
\end{align}
Since $n$, $(s,a)$ and $k$ were chosen arbitrarily from their respective sets, the result holds uniformly.
\end{proof}

%%%%%%%%%%%%%%%%%%%%%%%%%%%%%%%%%%%%%%%%%%%%%%%%%%%%%%%%%%%%%%%
\paragraph{Growth event pushes gap from $\bm{O(\alpha_k)}$ to $\bm{\delta}$}

% Let $N(k)$ denote the minimum number of steps required to accumulate a particular amount of learning rate mass:
Define the growth window $N(k)$ as
\begin{align}
N(k)
% :=
% \inf \left\{ i \ge 1 : 
% A(k,k+i)
% \ge \frac{4 \delta}{c_g} \right\} .
\:=
\left\lceil \frac{4 \delta}{c_g \rho_\alpha \alpha_k} \right\rceil . 
\end{align}

\begin{lemma}
\label{thm: growth window}
Under \autoref{asp: standing},
there exists a finite time $K_N$ such that for every $k \ge K_N$,
\[
N(k) \le \frac{1}{\alpha_{k}},
\qquad 
A(k,k + N(k)) \ge \frac{4 \delta}{c_g} .
\]
\end{lemma}
\begin{proof}
The definitions of $\delta$ and $c_g$ in \eqref{eq:delta-choice}--\eqref{eq:c-def} ensure that
\begin{multline*}
    \frac{4 \delta}{c_g \rho_\alpha}
    < 
    \frac{4 \rho_\alpha \sigma_{\min} d_{\min}}{32 \sigma_{\min}(d_{\min} - 2 \delta) \rho_\alpha}
    =
    \frac{d_{\min}}{8 (d_{\min} - 2 \delta)}
%    \\
    <
    \frac{d_{\min}}{8(d_{\min} - d_{\min}/4)}
    =
    \frac{1}{6}
    < 1 .
\end{multline*}
The Robbins--Monro conditions in \autoref{asp: standing}.(1) imply that 
$(\alpha_k)_{k \ge 0}$ decreases to 0. So there exists a finite $K$ such that $4\delta / c_g \rho_\alpha + \alpha_k < 1$ for all $k \ge K$,
and hence
\[
N(k) < \frac{4 \delta}{c_g \rho_\alpha \alpha_k} + 1 < \frac{1}{\alpha_k} .
\]
% Let
% \[
% r_k := \left\lceil \frac{4 \delta}{c_g \rho_\alpha \alpha_k} \right\rceil .
% \]
% Since $r_k \le 1/ \alpha_k$,
The comparability property of the learning rates in \autoref{asp: additional conditions on learning rate} then guarantees that $A(k, k+N(k)) \ge N(k) \rho_\alpha \alpha_{k} \ge 4 \delta / c_g$ for every $k \ge K_n := \max\{K, K_\alpha\}$.
% \[
% A(k, k+r_{k}) \ge r_{k} \rho_\alpha \alpha_{k} \ge 4 \delta / c_g .
% \]
% As a result, there exists a finite time $K_N \ge \max\{K, K_\alpha\}$ such that $N(k) \le r_k \le 1/\alpha_{k}$ for all $k \ge K_N$.
\end{proof}

For any $k \in \{t_n, \dots, t_{n+1}-1\}$,
% any state $s \in \mathcal S$
% and any action $a \in \mathcal A(s)$,
define the accumulated stochastic perturbations between times $k$ and $k'>k$ as
\[
\Xi(k, k' \mid s,a) 
:= 
\sum_{i=k}^{k'-1} \alpha_{i} \xi_i(s,a)
=
\sum_{i=k}^{k'-1} \alpha_{i} \big( w_i(s, \pi_{t_n}(s))-w_i(s,a) \big) ,
\]
% with the convention $\Xi(k, k \mid s,a) = 0$.
the growth event
% For any index $n \in \mathbb Z_{\ge 0}$,
% any time $k \in \{t_n, \dots, t_{n+1}-1\}$,
% any state $s \in \mathcal S$
% and any action $a \in \mathcal A(s)$,
as
\[
\mathcal E^g_k(s,a)
:=
\bigcap_{k'=k+1}^{k+N(k)}
\left\{
\Xi(k,k' \mid s,a) > - \frac{c_s}{2} \rho_\alpha M \alpha_k - \frac{c_g}{2} A(k,k')
\right\} ,
\]
% for any $M \in \mathbb N$,
and the hitting time as
\[
\tau^g_k(s,a) := \inf \! \big\{ k' \in \{k+1, \dots, k+N(k) \} : q_{k'}(s,\pi_{t_n}(s)) - q_{k'}(s,a) \ge \delta \big\} .
\]
We now show that, on $\mathcal E^g_k(s,a)$, 
perturbations cannot prevent ``bad'' gaps from reaching $\delta$ in $N(k)$ steps.
% significantly offset the mean-field drift.
% the gap $d_k(s,a)$ stays strictly positive until $\min\{\tau^g_k(s,a), t_{n+1}\}$, and reaches $\delta$ within $N(k)$ steps, meaning

\begin{lemma}[Growth event]
\label{thm: from a_k margin to delta}
Consider the setting of \autoref{thm: fixed prob that bad actions dont become greedy}
and let $K_\xi$ and $K_N$ be the finite times obtained in \autoref{thm: properties noise gap} and \autoref{thm: growth window}, respectively.
For every index $n \in \mathbb Z_{\ge 0}$ such that $t_n \ge K_N$,
every pair $(s,a) \in \mathcal B_{\pi_{t_n}}$
and every time $k \in \{t_n, \dots, t_{n+1}-1\}$,
the joint event
\begin{align}
\label{eq: joint growth event}
\mathcal E^g_k(s,a) \cap \big\{ c_s \rho_\alpha M \alpha_k \leq d_k(s,a) < \delta \big\}
% \mathcal E^g_k(s,a) \cap \big\{ c_s \rho_\alpha M \alpha_k \leq q_{k}(s,\pi_{t_n}(s)) - q_{k}(s,a) < \delta \big\}
\end{align}
guarantees that 
% the gap $d_k(s,a)$ stays strictly positive until $\min\{\tau^g_k(s,a), t_{n+1}\}$, and reaches $\delta$ within $N(k)$ steps, meaning
% for every $k'$ satisfying
% $k \le k' \le \min \{ \tau^g_k(s,a), t_{n+1} \}$,
\begin{align}
\label{eq: growth event gap remains positive}
    d_{k'}(s,a) > 0
    % q_{k'}(s,\pi_{t_n}(s)) - q_{k'}(s,a) > 0,
    \qquad \forall \, k \le k' \le \min \{ \tau^g_k(s,a), t_{n+1} \} , 
\end{align}
and
\begin{align}
\label{eq: hit time event a_k to delta}
    \min \{ \tau^g_k(s,a) , t_{n+1} \} \le k + N(k) .
\end{align}
Moreover, if $t_n \ge K_\xi$,
\begin{align}
\label{eq: prob from a_k reach delta}
    \mathbb P (\mathcal E^g_k(s,a) \mid \mathcal F_k, \pi_{t_n})
    \ge 1 - \frac{8 c_\xi}{c^2 \rho_\alpha^2} \sum_{j = 0}^{\infty} \frac{2^j}{(M + 2^j)^2} .
\end{align}
\end{lemma}
\begin{proof}
Fix an $n \in \mathbb Z_{\ge 0}$ such that $t_n \ge K_N$,
and let $\pi := \pi_{t_n}$.
Fix a pair $(s,a) \in \mathcal B_\pi$,
define the sequence $\big( d_k(s,a) \big)_{k=t_n}^{t_{n+1}}$ as in \autoref{def: gap variables}, and write $d_k := d_k(s,a)$, $d_\pi := d_\pi(s,a)$ and $\xi_k := \xi_k(s,a)$ for brevity.
Fix a time $k \in \{t_n, \dots, t_{n+1}-1\}$ and write $\tau^g_k := \tau^g_k(s,a)$.

For every $k' \in \{k, \dots, \tau^g_k\}$ we have $d_{k'} \le \delta$, and hence
\[
\sigma_{k'}(s) (d_\pi - d_{k'}) \ge \sigma_{\min}(d_{\min} - \delta)
>
\sigma_{\min}(d_{\min} - 2 \delta)
= c_g .
\]
Thus summing \eqref{eq: def gap evolution} from $k$ to any $k'$ satisfying $k < k' \le \min\{\tau^g_k, t_{n+1}\}$ yields
\[
d_{k'}
>
d_{k} + c_g A(k,k') + \Xi(k,k' \mid s,a) .
\]
% $k' \in \{k, \dots, t_{n+1}-1\}$
% the gap satisfies
% \[
% d_{k'+1}
% =
% d_{k'} + \alpha_{k'} \big( \mu_{k'}(s) (d_\pi - d_{k'}) + \xi_{k'} \big) ,
% \]
% where the definition of $\delta$ guarantees that $d_\pi \ge d_{\min} > 2 \delta$,
% and the strict exploring starts ensure that $\mu_{k'}(s) \ge \sigma_{\min}$. 
% Since $d_{k'} < \delta$ before the hitting time $\tau^g_k$,
% we can bound the update for all $k \le k' < ( \tau^g_k \wedge t_{n+1} )$ as
% \begin{align*}
% d_{k'+1}
% >
% d_{k'} + \alpha_{k'} \left( \sigma_{\min}(d_{\min} - \delta) + \xi_{k'} \right)
% % \\
% % &>
% % d_{k'} + \alpha_{k'} \left( \sigma_{\min}(d_{\min} - 2 \delta) + \xi_{k'} \right)
% % \\
% >
% d_{k'} + \alpha_{k'} \left( c_g + \xi_{k'} \right) ,
% \end{align*}
% using $c_g := \sigma_{\min} ( d_{\min} - 2 \delta)$.
% Summing from $k$ to $k'$ then yields
% \[
% d_{k'}
% >
% d_{k} + c_g A(k,k') + \Xi(k,k' \mid s,a) .
% \]
As a result, the joint event in \eqref{eq: joint growth event} implies that as long as $k < k' \le \min\{ \tau^g_k , t_{n+1}, k+N(k) \}$,
\begin{align}
d_{k'}
&>
c_s \rho_\alpha M \alpha_k + c_g A(k,k') - \frac{c_s}{2} \rho_\alpha M \alpha_k - \frac{c_g}{2} A(k,k')
\nonumber
\\
\label{eq: gap remains positive up to tau'_k}
&=
\frac{c_s}{2} \rho_\alpha M \alpha_k + \frac{c_g}{2} A(k,k') > 0 .
\end{align}
% which proves \eqref{eq: growth event gap remains positive}.
% 
Now assume that $t_{n+1} > k+N(k)$, otherwise $\min\{ \tau^g_k, t_{n+1} \} \le k+N(k)$ is immediate.
Suppose for contradiction that $\tau^g_k = \infty$ because $d_k$ does not reach $\delta$ within $N(k)$ steps. 
Since $k \ge t_n \ge K_N$, \autoref{thm: growth window} and equation \eqref{eq: gap remains positive up to tau'_k} yield
\[
d_{k+N(k)}
>
\frac{c_s}{2} \rho_\alpha M \alpha_k + \frac{c_g}{2} A(k,k+N(k)) 
>
% \frac{c_g}{2} A(k,k+N(k)) 
% \ge
\frac{c_g}{2} \frac{4 \delta}{c_g} 
> 
\delta .
\]
This contradicts $\tau^g_k = \infty$. Hence $\tau^g_k \le k+N(k)$, which proves \eqref{eq: hit time event a_k to delta}; and with \eqref{eq: gap remains positive up to tau'_k} also gives \eqref{eq: growth event gap remains positive}.

% This contradiction implies that $\tau^g_k \le k+N(k)$, which proves \eqref{eq: hit time event a_k to delta} and, together with \eqref{eq: gap remains positive up to tau'_k}, proves \eqref{eq: growth event gap remains positive}.

% Then \eqref{eq: gap remains positive up to tau'_k} yields
% \[
% d_{k+N(k)}
% >
% \frac{c_s}{2} \rho_\alpha M \alpha_k + \frac{c_g}{2} A(k,k+N(k)) 
% >
% \frac{c_g}{2} A(k,k+N(k)) .
% \]
% However,

% since the definition of $N(k)$ guarantees that $A(k,k+N(k)) \ge 4 \delta/c_g$, we obtain a contradiction: $d_{k+N(k)}
% >
% \frac{c_s}{2} \rho_\alpha M \alpha_k + 2 \delta
% > \delta .$
% % \[
% % d_{k+N(k)}
% % >
% % \frac{c_s}{2} \rho_\alpha M \alpha_k + 2 \delta
% % > \delta .
% % \]
% Thus, $\tau^g_k \le k+N(k)$, which proves \eqref{eq: hit time event a_k to delta}.

Finally,
% we lower-bound the probability of $\mathcal E^g_k(s,a)$.
the complement event $\overline{\mathcal E^g_k(s,a)}$ implies that for some $i \in \{1, \dots , N(k) \}$,
\begin{align*}
\bigl| \Xi(k,k+i \mid s,a) \bigr|
\ge \frac{c_s}{2} \rho_\alpha M \alpha_k + \frac{c_g}{2} A(k,k+i)
% \]
% Using $c := \min\{c_s,c_g\}$ and the comparability property (since $i \le N(k) \le 1/\alpha_k$ for $k \ge t_n \ge K_N$), this simplifies to
% \[
% \bigl| \Xi(k,k+i \mid s,a) \bigr|
\ge \frac{c}{2} \big( \rho_\alpha M \alpha_k + A(k,k+i) \big) .
% \ge \frac{c}{2} \rho_\alpha \alpha_k ( M + i ).
\end{align*}
Choosing $j \in \mathbb N$ such that $2^j \le i \le 2^{j+1}$ 
and using the comparability of the learning rates from \autoref{asp: additional conditions on learning rate}
yields the bound
\begin{align}
    \max_{1 \le \ell \le \min \{ 2^{j+1} , N(k) \}}
    |\Xi(k,k + \ell \mid s,a)|
    \ge \frac{c}{2} \big( \rho_\alpha M \alpha_k + A(k,k + 2^j) \big)
    \ge \frac{c}{2} \rho_\alpha \alpha_k (M + 2^j) .
\end{align}
So if $t_n \ge K_\xi$,
% By \autoref{thm: properties noise gap}, there exists an almost surely finite $K_\xi$ such that 
the sequence $\big(\Xi(k, k+\ell \mid s,a)\big)_{\ell \ge 1}$ is a square-integrable martingale adapted to the filtration $(\mathcal F_{k+\ell})_{\ell \ge 1}$. So, by Doob's martingale inequality,
\begin{align}
\label{eq: prob of not E^g_k}
    \mathbb P
    \left( 
    \overline{\mathcal E^g_k(s,a)} \mid \mathcal F_{k}, \pi_{t_n}
    \right)
    % \\
    % \nonumber
    \le
    \sum_{j = 0}^{\infty} \frac{\mathbb E \left[ \big( \Xi(k, \min\{ k+2^{j+1}, k+N(k) \} \mid s,a) \big)^2 \mid \mathcal F_{k}, \pi_{t_n} \right]}{\left( \dfrac{c}{2} \rho_\alpha \alpha_k ( M + 2^j) \right)^2} .
\end{align}
\autoref{thm: properties noise gap} and the non-increasing sequence $(\alpha_k)_{k \ge 0}$ imply that
\begin{align*}
\mathbb E \left[ \big( \Xi(k, \min\{ k+2^{j+1}, k+N(k)\} \mid s,a) \big)^2 \mid \mathcal F_{k}, \pi_{t_n} \right]
\le
c_\xi \sum_{\ell=0}^{2^{j+1}-1} \alpha^2_{k+\ell}
\le
c_\xi 2^{j+1} \alpha^2_k .
\end{align*}
Substituting this back into \eqref{eq: prob of not E^g_k} proves \eqref{eq: prob from a_k reach delta} because
\begin{align}
    \mathbb P
    \left( 
    \overline{\mathcal E^g_k(s,a)} \mid \mathcal F_{k}, \pi_{t_n}
    \right)
    \le
    \frac{8 c_\xi}{c^2 \rho_\alpha^2} \sum_{j = 0}^{\infty} \frac{2^j}{(M+ 2^j)^2}.
\end{align}
Since $n$, $(s,a)$ and $k$ were chosen arbitrarily from their respective sets, the result holds uniformly.
\end{proof}

% \autoref{thm: from a_k margin to delta} confirms that on the event $\mathcal E^g_k \cap \{c \rho_\alpha M \alpha_k \leq q_k(s,\pi_{t_n}(s))-q_k(s,a) < \delta \}$, the gap safely hits $\delta$ in at most $N(k)$ steps while remaining strictly positive throughout the ascent.

%%%%%%%%%%%%%%%%%%%%%%%%%%%%%%%%%%%%%%%%%%%%%%%%%%%%%%%%
\paragraph{Lock-in event keeps gap positive after reaching $\bm{\delta}$}

Define the lock-in event as
\begin{align}
\label{eq: definition lock in event single gap}
    \mathcal E^l_{t_n}(s,a) := \left\{ \sup_{k \ge t_n} 
    % \bigl| \Xi(t_n, k \mid s,a) \bigr| 
    \Bigg| \sum_{i=t_n}^{k-1} \alpha_{i} \big( w_i(s, \pi_{t_n}(s))-w_i(s,a) \big) \Bigg|
    < \frac{\delta}{8} \right\} .
\end{align}
We now show that, on $\mathcal E^l_{t_n}(s,a)$,
once a ``bad'' gap reaches $\delta$, 
it remains strictly positive indefinitely.
% perturbations can never drive it below zero anymore.
% For any $n \in \mathbb Z_{\ge 0}$,
% any state $s \in \mathcal S$
% and any action $a \in \mathcal A(s)$,

% \textcolor{red}{need extension of definition so that $\xi_k$ is defined to $\infty$}

\begin{lemma}[Lock-in event]
\label{thm: event noise remains small}
Consider the setting of \autoref{thm: fixed prob that bad actions dont become greedy} and let $K_\xi$ be the finite time obtained in \autoref{thm: properties noise gap}.
For every index $n \in \mathbb Z_{\ge 0}$,
every pair $(s,a) \in \mathcal B_{\pi_{t_n}}$
and every time $k \in \{t_n, \dots, t_{n+1}-1\}$,
the joint event
\[
\mathcal E^l_{t_n}(s,a) \cap 
% \big\{ q_{k}(s,\pi_{t_n}(s)) - q_{k}(s,a) \ge \delta \big\}
\big\{ d_k(s,a) \ge \delta \big\}
\]
guarantees that
% the gap remains strictly positive until the next transition time, meaning
% for every $k' \in \{k, \dots, t_{n+1}\}$,
\begin{align}
\label{eq: gap remains strictly positive}
% q_{k'}(s,\pi_{t_n}(s)) > q_{k'}(s,a) 
d_{k'}(s,a) > 0
\qquad 
\forall \, k' \in \{k, \dots, t_{n+1}\} .
\end{align}
Moreover, if $t_n \ge K_\xi$,
% there exists an almost surely finite time $K_\xi$ such that for every $t_n \ge K_\xi$,
\begin{equation}
\label{eq:probability-small-remaining-noise}
\mathbb P\!\left(\mathcal E^l_{t_n}(s,a)\mid \mathcal F_{t_n}, \pi_{t_n} \right)
\ge
1-\frac{64 c_\xi}{\delta^2}\sum_{i=t_n}^\infty \alpha_i^2 .
\end{equation}
\end{lemma}
\begin{proof}
Fix an $n \in \mathbb Z_{\ge 0}$
and let $\pi := \pi_{t_n}$.
Fix a pair $(s,a) \in \mathcal B_\pi$,
define the sequence $\big( d_k(s,a) \big)_{k=t_n}^{t_{n+1}}$ as in \autoref{def: gap variables}, and write $d_k := d_k(s,a)$, $d_\pi := d_\pi(s,a)$ and $\xi_k := \xi_k(s,a)$ for brevity.
Fix a time $k \in \{t_n, \dots, t_{n+1}-1\}$.

% For every time $k \in \{t_n, \dots, t_{n+1}-1\}$, the gap is updated as
% \begin{align*}
% d_{k+1}
% =
% \big( 1 - \alpha_{k} \mu_{k}(s) \big) d_{k} + \alpha_{k} \mu_{k}(s) d_\pi + \alpha_{k} \xi_{k} .
% \end{align*}
% To unroll this recursion, define the contraction factor between steps $k$ and $k'>k$ as
% \[
% \Gamma(k,k') := \prod_{i=k}^{k'-1} \big( 1 - \alpha_{i} \mu_{i}(s) \big) \in(0,1].
% \]
% with the convention that $\Gamma(k,k) = 1$.

Define the contraction factor between times $k$ and $k'>k$ as
\[
\Gamma(k,k') := \prod_{i=k}^{k'-1} \big( 1 - \alpha_{i} \mu_{i}(s) \big) \in(0,1] 
\]
with the convention that $\Gamma(k,k) = 1$.
Unrolling \eqref{eq: def gap evolution} from $k$ to any $k' \in \{k,\dots,t_{n+1}\}$ yields
\begin{align}
\nonumber
d_{k'}
&=
\Gamma(k,k') d_{k}
+
\bigl(1-\Gamma(k,k') \bigr) d_\pi
+ \sum_{i=k}^{k'-1} \Gamma(i+1,k') \alpha_i \xi_i
\\
\label{eq: induction for abel proof}  
&\ge
\delta + \sum_{i=k}^{k'-1} \Gamma(i+1,k') \alpha_i \xi_i
\end{align}
because $\Gamma(k,k') \in(0,1]$,
$d_\pi \ge d_{\min} > \delta$
% $\Gamma(k,k') \in (0,1]$,
and, by assumption, $d_{k} \ge \delta$.
% we have
% \begin{align}
% \label{eq: lower obund for abel proof}
% \Gamma(k,k') d_{k} + \big(1-\Gamma(k,k')\big) d_\pi \ge \delta.
% \end{align}
% To bound the stochastic part, we note that for a fixed terminal index $k'$, 
Since the map $i \mapsto \Gamma(i+1,k')$ is non-decreasing with values in $(0,1]$, 
using summation by parts, the triangle inequality and equation \eqref{eq: definition lock in event single gap} yield
\begin{align}
\nonumber
\left|
\sum_{i=k}^{k'-1} \Gamma(i+1,k') \alpha_i \xi_i
\right|
&\le
\big| \Gamma(k',k') \big| \big| \Xi(k,k') \big| + \sum_{i=k}^{k'-2} \big| \Gamma(i+2,k') - \Gamma(i+1,k') \big| \big| \Xi(k, i+1) \big|
\\
\nonumber
&\le
\left( \Gamma(k',k') + \sum_{i=k}^{k'-2} \big( \Gamma(i+2,k') - \Gamma(i+1,k') \big) \! \right)
\sup_{k \le m \le k'} \bigl| \Xi(k,m\mid s,a) \bigr|
% \\
% \nonumber
% &=
% \big( \Gamma(k',k') + \Gamma(k',k') - \Gamma(k+1,k') \big)
% \sup_{k \le m \le k'} \bigl| \Xi(k,m\mid s,a) \bigr|
\\
\nonumber
&\le
2 \sup_{k \le m \le k'} \bigl| \Xi(k,m\mid s,a) \bigr|
\\
\label{eq:abel equation on noise}
&\le
2 \left( \bigl|\Xi(t_n,k \mid s,a)\bigr|
+
\sup_{k \le m \le k'} \bigl|\Xi(t_n, m \mid s,a)\bigr| \right)
<
\frac{\delta}{2}.
\end{align}
% On $\mathcal E^l_{t_n}(s,a)$, the triangle inequality also implies that for every $m > k$,
% \begin{align}
% \label{thm: triangle on Xi}
% \bigl|\Xi(k,m\mid s,a)\bigr|
% % &=
% % \bigl|
% % \Xi(t_n , m \mid s,a) - \Xi(t_n, k \mid s,a)
% % \bigr|
% \le
% \bigl|\Xi(t_n,k\mid s,a)\bigr|
% +
% \bigl|\Xi(t_n,m\mid s,a)\bigr|
% <
% \frac{\delta}{4}.
% \end{align}
Hence, \eqref{eq: induction for abel proof} and
\eqref{eq:abel equation on noise}
show that $d_{k'} \ge \delta - \delta/2 > 0$ for every $k' \in \{k,\dots,t_{n+1}\}$,
proving \eqref{eq: gap remains strictly positive}.
% This strictly limits the noise term in \eqref{eq: induction for abel proof} by
% \begin{equation}
% \label{eq:shifted-noise-bound}
% \left|
% \sum_{i=k}^{k'-1} \Gamma(i+1,k') \alpha_i \xi_{i}
% \right|
% < \frac{\delta}{2}.
% \end{equation}
% Plugging \eqref{eq: lower obund for abel proof} and \eqref{eq:shifted-noise-bound} into \eqref{eq: induction for abel proof}  
% yields
% \[
% d_{k'}
% >
% \delta - \frac{\delta}{2}
% >
% 0
% \qquad
% \forall \, k' \in \{k,\dots,t_{n+1}\} .
% \]
% for every $k' \in \{k,\dots,t_{n+1}\}$, which proves \eqref{eq: gap remains strictly positive}.

Finally, 
% we lower-bound the probability of $\mathcal E^l_{t_n}(s,a)$.
% We observe that 
the complement event $\overline{\mathcal E^l_{t_n}(s,a)}$
% implies that for some time $k \ge t_{n}$,
% \[
% \bigl| \Xi(t_n, k \mid s,a)\bigr|
% \ge \frac{\delta}{8} ,
% \]
% or equivalently
satisfies
\[
% \left\{\sup_{k\ge t_n}|\Xi(t_n,k\mid s,a)|\ge \frac{\delta}{8}\right\}
\overline{\mathcal E^l_{t_n}(s,a)}
\equiv
\bigcup_{t = t_n}^\infty
\left\{
\max_{t_n\le k\le t}|\Xi(t_n,k\mid s,a)|
\ge \frac{\delta}{8}
\right\} \! .
\]
If $t_n \ge K_\xi$,
% \autoref{thm: properties noise gap} indicates that there exists an almost surely finite time $K_\xi$ such that
the sequence $\big( \Xi(t_n, k \mid s,a)\big)_{t_n\le k \le t}$ is a square-integrable martingale adapted to the filtration $(\mathcal F_{k})_{t_n \le k\le t}$.
Hence, Doob's martingale inequality and the monotone convergence of conditional probabilities yield
\begin{align}
\nonumber
\mathbb P\!\left(
\overline{\mathcal E^l_{t_n}(s,a)}
\,\middle|\,
\mathcal F_{t_n},\pi_{t_n}
\right)
&=
\lim_{t \to \infty}
\mathbb P\!\left(
\max_{t_n \le k \le t}| \Xi(t_n, k\mid s,a) |
\ge \frac{\delta}{8}
\,\middle|\,
\mathcal F_{t_n},\pi_{t_n}
\right) \! .
% \\
% \label{eq:hold-doob-trunc}
% &\le
% \frac{64}{\delta^2}
% \limsup_{m\to\infty}
% \mathbb E\!\left[M_m^2\mid \mathcal F_{t_n},\pi_{t_n}\right].
\\
\label{eq: prob of lock in in proof}
&\le
\lim_{t \to \infty}
\frac{64}{\delta^2}
\mathbb E\!\left[\Xi(t_n,t\mid s,a)^2 \mid \mathcal F_{t_n}, \pi_{t_n} \right] \! .
\end{align}
% So if $t_n \ge K_\xi$,
% % \autoref{thm: properties noise gap} indicates that there exists an almost surely finite time $K_\xi$ such that
% the sequence $\big\{\Xi(t_n, k \mid s,a)\big\}_{t_n\le k \le t}$ is a square-integrable martingale adapted to the filtration $\{\mathcal F_{k}\}_{t_n \le k\le t}$.
% Hence, Doob's martingale inequality yields
% \begin{align}
% \label{eq: prob of lock in in proof}
% \mathbb P\!\left(
% \max_{t_n \le k \le t} \bigl| \Xi(t_n, k \mid s,a)\bigr|
% \ge \frac{\delta}{8}
% \,\middle|\,
% \mathcal F_{t_n}, \pi_{t_n}
% \right)
% \le
% \frac{64}{\delta^2}
% \mathbb E\!\left[\Xi(t_n,t\mid s,a)^2\mid \mathcal F_{t_n}, \pi_{t_n} \right] \! .
% \end{align}
% Using the orthogonality of martingale increments and the variance bound from \eqref{eq:xi-var}, we obtain
Using the tower property and \autoref{thm: properties noise gap}, we obtain
\begin{align*}
\mathbb E\!\left[\Xi(t_n, t \mid s,a)^2\mid \mathcal F_{t_n}, \pi_{t_n} \right]
&=
\sum_{i=t_n}^{t-1} \alpha_{i}^2\,
\mathbb E\!\left[\xi_{i}^2\mid \mathcal F_{t_n}, \pi_{t_n} \right]
\\
&=
\sum_{i=t_n}^{t-1} \alpha_{i}^2\,
\mathbb E\!\left[
\mathbb E\!\left[\xi_{i}^2\mid \mathcal F_{i}, \pi_{t_n}\right]
\middle|\mathcal F_{t_n} , \pi_{t_n}
\right]
\\
&\le
c_\xi \sum_{i=t_n}^{t-1} \alpha_{i}^2
% \\
% &\le
% c_\xi \sum_{i=t_n}^{\infty} \alpha_{i}^2 .
\end{align*}
Substituting this back in \eqref{eq: prob of lock in in proof} proves \eqref{eq:probability-small-remaining-noise} because
\[
\mathbb P\!\left(
\overline{\mathcal E^l_{t_n}(s,a)}
\,\middle|\,
\mathcal F_{t_n}, \pi_{t_n}
\right)
\le
\frac{64 c_\xi}{\delta^2}\sum_{i=t_n}^{\infty} \alpha_{i}^2 .
\]
Since $n$, $(s,a)$, and $k$ were chosen arbitrarily from their respective sets, the result holds uniformly.
\end{proof}

% In other words, for sufficiently large $k$, whenever a gap $d^b_k$ reaches a fixed $\delta > 0$, it remains positive with large probability. 

%%%%%%%%%%%%%%%%%%%%%%%%%%%%%%%%%%%%%%%%%%%%%%%%%%%%%%%%%%%%%%%
\paragraph{Global event for all gaps to remain strictly positive}
% so that none of them become greedy up to and including the transition time $t_{n+1}$.
Fix an $n \in \mathbb Z_{\ge 0}$,
let $L := |\mathcal B_{\pi_{t_n}}|$,
and assume that $L>0$, otherwise \autoref{thm: fixed prob that bad actions dont become greedy} is immediate.
% We now apply the three sequential events to every pair in $\mathcal B_{\pi_{t_n}}$.
Conditioned on $\{\mathcal F_{t_n}, \pi_{t_n}\}$, enumerate
\[
\mathcal B_{\pi_{t_n}} = \big\{ (s_1,a_1), \dots, (s_L,a_L) \big\} .
\]
We first apply the $M$-step seed block of \autoref{thm: event to reach a_k margin} to $(s_\ell,a_\ell)$ at time $b_\ell := t_n + (\ell-1)M$ for each $\ell \in \{1,\dots,L\}$.
At the end of this seed phase, at time $T_n := t_n + LM$, we apply the growth event of \autoref{thm: from a_k margin to delta} simultaneously to all pairs. The lock-in event of \autoref{thm: event noise remains small} is enforced throughout.
This yields the global events
\begin{align}
\label{eq: def global events}
\mathcal E^s_{t_n} &:= \bigcap_{\ell=1}^{L} \mathcal E^s_{b_\ell}(s_\ell, a_\ell),
\qquad
\mathcal E^g_{t_n} := \bigcap_{\ell=1}^{L} \mathcal E^g_{T_n}(s_\ell, a_\ell),
\qquad
\mathcal E^l_{t_n} := \bigcap_{\ell=1}^{L} \mathcal E^l_{t_n}(s_\ell, a_\ell).
\end{align}

% We first apply the $M$-step seed block from \autoref{thm: event to reach a_k margin} to pair $(s_\ell,a_\ell)$
% at time $b_\ell := t_n + (\ell-1) M$ for $\ell \in \{1, \dots, L\}$.

% At the end of the seed phase, i.e. at time $T_n := t_n + L M$,
% we apply the $N(T_n)$-step growth block of \autoref{thm: from a_k margin to delta} to each pair.

% Finally, the global lock-in event of \autoref{thm: event noise remains small} ensures

% % to keep them strictly positive until time $t_{n+1}$.

% Altogether, we obtain the global seed, growth and lock-in events
% \begin{align}
% \label{eq: def global events}
% % \label{eq: def global seed event}
% \mathcal E^s_{t_n} &:= \bigcap_{\ell=1}^{L} \mathcal E^s_{b_\ell}(s_\ell, a_\ell) ,
% % \\
% \qquad
% % \label{eq: def global grow event}
% \mathcal E^g_{t_n} &:= \bigcap_{\ell=1}^{L} \mathcal E^g_{T_n}(s_\ell, a_\ell) ,
% % \\
% \qquad
% % \label{eq: def global hold event}
% \mathcal E^l_{t_n} &:= \bigcap_{\ell=1}^{L} \mathcal E^l_{t_n}(s_\ell, a_\ell) .
% \end{align}

The gap of a pair $(s_\ell, a_\ell)$ may remain zero while $k \le b_\ell$.
During this period, the policy $\pi_k$ may be resampled among greedy actions, and the tie can be broken in favour of $(s_\ell, a_\ell)$.
To exclude this behaviour, we introduce a global tie-breaking event that prevents any ``bad'' action from being selected while its gap is zero:
\begin{align}
\label{eq: def global tie event}
\mathcal E^t_{t_n} :=
\bigcap_{k=t_n}^{\min \{ T_n-1 , t_{n+1} \}}
\left\{
\forall \, s \in \mathcal S : 
(s, \pi_k(s)) \notin \mathcal B_{\pi_{t_n}}
\right\} .
\end{align}

% We now show that, on $\mathcal E^t_{t_n} \cap \mathcal E^s_{t_n} \cap \mathcal E^g_{t_n} \cap \mathcal E^l_{t_n}$
% no pair in $\mathcal B_{\pi_{t_n}}$ is selected up to and including time $t_{n+1}$.
% the greedy policy at time $t_{n+1}$ does not correspond to any pair in $\mathcal B_{\pi_{t_n}}$.

\begin{lemma}
\label{thm: event that yields positive gaps}
Consider the setting of \autoref{thm: fixed prob that bad actions dont become greedy} and let $K_N$ be the finite time obtained in \autoref{thm: growth window}.
For every index $n \in \mathbb Z_{\ge 0}$ such that $t_n \ge K_N$,
the joint event $\mathcal E^t_{t_n} \cap \mathcal E^s_{t_n} \cap \mathcal E^g_{t_n} \cap \mathcal E^l_{t_n}$ guarantees that
% the greedy policy $\pi_k$ never coincides with a bad action, 
% meaning
% for every time $k \in \{t_n, \dots, t_{n+1}\}$
% and every state $s \in \mathcal S$,
\begin{align}
\label{eq: proof combined event}
(s,\pi_{k}(s)) \notin \mathcal B_{\pi_{t_n}}
\qquad
\forall \, s \in \sset, \
\forall \, k \in \{t_n, \dots, t_{n+1}\} .
\end{align}
\end{lemma}

\begin{proof}
Fix an $n \in \mathbb Z_{\ge 0}$ such that $t_n \ge K_N$, and let $\pi := \pi_{t_n}$.
For each pair $(s,a) \in \mathcal B_\pi$, define the sequence $\big( d_k(s,a) \big)_{k=t_n}^{t_{n+1}}$ as in \autoref{def: gap variables}.
% We track the evolution of a specific gap $d_k(s_\ell, a_\ell)$ across three phases.
We show that, on the joint event, no pair in $\mathcal B_{\pi}$ is selected up to and including time $t_{n+1}$.

Fix a pair $(s_\ell,a_\ell) \in \mathcal B_\pi$.
At time $t_n$, all gaps are non-negative as $\pi_{t_n}$ is greedy
with respect to $q_{t_n}$.
Before the seed block of $(s_\ell,a_\ell)$ starts, earlier seed blocks
cannot decrease its gap.
Indeed, if the seed block of a pair $(s_j,a_j)$ with $j<\ell$ updates
the pair $(s_j,a)$, then the case distinction in
\eqref{eq: def good prefix} gives
$(s_j,a) \neq (s_\ell,a_\ell)$.
Thus, if $s_j \neq s_\ell$, the gap $d_k(s_\ell,a_\ell)$ is unchanged;
if $s_j = s_\ell$ and $a \neq \pi(s_\ell)$, it is also unchanged; 
and if $s_j=s_\ell$ and $a=\pi(s_\ell)$, it increases, as in
\eqref{eq: final lower bound case 1} and \eqref{eq: increase in d^b case B}.
Hence, by time $b_\ell$, either the gap has already reached $\delta$,
or
\[
0 \le d_{b_\ell}(s_\ell,a_\ell) < \delta .
\]

If the gap has already reached $\delta$, then
$\mathcal E^l_{t_n}(s_\ell,a_\ell)$ keeps it strictly positive up to
time $t_{n+1}$ by \autoref{thm: event noise remains small}.
Otherwise, applying \autoref{thm: event to reach a_k margin} at time
$b_\ell$ on $\mathcal E^s_{b_\ell}(s_\ell,a_\ell)$ yields
\[
d_k(s_\ell,a_\ell) > c_s A(b_\ell,k) > 0
\qquad
\forall \, b_\ell < k \le \min\{ \tau^s_{b_\ell}(s_\ell,a_\ell), b_\ell+M, t_{n+1} \}.
\]
Thus during its seed block, the gap either reaches $\delta$, or, at the end of the block,
\[
d_{b_\ell+M}(s_\ell,a_\ell)
> c_s A(b_\ell,b_\ell+M)
\ge c_s M \alpha_{b_\ell+M}
\ge c_s M \alpha_{T_n}
\ge c_s \rho_\alpha M \alpha_{T_n}.
\]
Here we used that $(\alpha_k)_{k\ge0}$ is non-increasing and that
$\rho_\alpha \in (0,1]$.

Later seed blocks cannot decrease this gap by the same argument as above. Therefore, at time $T_n$, either the gap has already reached $\delta$, in which case the lock-in event keeps it strictly positive up to $t_{n+1}$, or else
\[
d_{T_n}(s_\ell,a_\ell)
> c_s \rho_\alpha M \alpha_{T_n}.
\]
In the latter case. Since $T_n \ge t_n \ge K_N$,
applying \autoref{thm: from a_k margin to delta} on $\mathcal E^g_{T_n}(s_\ell,a_\ell)$
ensures that the gap remains strictly positive until it reaches $\delta$,
or until $t_{n+1}$ if the transition occurs first.
If the gap reaches $\delta$ before $t_{n+1}$, the lock-in event again keeps it strictly positive up to $t_{n+1}$.

Consequently, for every $(s,a) \in \mathcal B_\pi$ and every
$k \in \{t_n,\dots,t_{n+1}\}$, either the gap $d_k(s,a)$ is zero during the seed phase, in which case $\mathcal E^t_{t_n}$ prevents the policy from selecting $(s,a)$, 
or the gap is strictly positive,
in which case $(s,a)$ cannot be greedy because
\[
q_k(s,\pi_k(s))
\ge q_k(s,\pi(s))
> q_k(s,a) .
\]
% If $t_{n+1}<T_n$, this conclusion follows from
% $\mathcal E^t_{t_n} \cap \mathcal E^s_{t_n} \cap \mathcal E^l_{t_n}$.
% If $t_{n+1}\ge T_n$, it follows from $\mathcal E^g_{t_n} \cap \mathcal E^l_{t_n}$.
% Therefore, combining the events yields
% \[
% (s,\pi_k(s)) \notin \mathcal B_{\pi_{t_n}}
% \qquad
% \forall\, s \in \mathcal S , \
% \forall\, k \in \{t_n,\dots,t_{n+1}\} .
% \]
% This proves the claim.
Either way, the joint event $\mathcal E^t_{t_n} \cap \mathcal E^s_{t_n} \cap \mathcal E^g_{t_n} \cap \mathcal E^l_{t_n}$ guarantees that
\[
(s,\pi_k(s)) \notin \mathcal B_{\pi_{t_n}}
\qquad
\forall\, s \in \mathcal S , \
\forall\, k \in \{t_n,\dots,t_{n+1}\} .
\]
which proves \eqref{eq: proof combined event}.
\end{proof}

%%%%%%%%%%%%%%%%%%%%%%%%%%%%%%%%%%%%%%%%%%%%%%%%%%%%%%%%
\paragraph{Proof of \autoref{thm: fixed prob that bad actions dont become greedy}}

% \begin{proof}[\autoref{thm: fixed prob that bad actions dont become greedy}]
Define the constants
\begin{align}
p_t &:= \beta^{|\mathcal S| } \in (0,1) ,
\\
p_s &:= \sigma_{\min} \frac{d_{\min}^2}{d_{\min}^2 + 16 c_v} \in (0,1) ,
\\
p_g
&:=
1-|\mathcal S\times\mathcal A|\,
\frac{8 c_\xi}{c^2\rho_\alpha^2}
\sum_{j=0}^{\infty}\frac{2^j}{(M+2^j)^2} ,
\\
p_{\overline l} &:= 
|\mathcal S\times\mathcal A|
\frac{64 c_\xi}{\delta^2}\sum_{i=t_n}^\infty \alpha_i^2 .
\end{align}
Since $\lim_{M \to \infty} \sum_{j \ge 0} 2^j/(M+2^j)^2 = 0$, we can choose $M \in \mathbb N$ sufficiently large such that $p_g \ge 1/2$.
The Robbins--Monro conditions in \autoref{asp: standing}.(1) also ensure that there exists a \(K_l<\infty\) such that $p_{\overline l} \le \frac14 ( p_t p_s ) ^{|\mathcal S\times\mathcal A| M}$ for all \(t_n \ge K_l\).

Fix an index $n \in \mathbb Z_{\ge 0}$ such that $t_n \ge \max\{K_\xi, K_N, K_l\}$
where $K_\xi$ and $K_N$ be the finite times obtained in \autoref{thm: properties noise gap} and \autoref{thm: growth window}, respectively.
Let $L := |\mathcal B_{\pi_{t_n}}|$ and assume that $L>0$, otherwise \autoref{thm: fixed prob that bad actions dont become greedy} is immediate.

From \autoref{thm: event that yields positive gaps} and $\mathcal E^t_{t_n} \cap \mathcal E^s_{t_n} \in \mathcal F_{T_n}$, it follows that
\begin{align}
\nonumber
&\mathbb P \big(
\forall \, s \in \sset, \
\forall \, k \in \{t_n, \dots, t_{n+1}\}
:
(s,\pi_{k}(s)) \notin \mathcal B_{\pi_{t_n}}
\mid \mathcal F_{t_n}, \pi_{t_n}
\big)
\\
\nonumber
&\qquad\ge
\mathbb P ( \mathcal E^t_{t_n} \cap \mathcal E^s_{t_n} \cap \mathcal E^g_{t_n} \cap \mathcal E^l_{t_n} \mid \mathcal F_{t_n}, \pi_{t_n} )
\\
\nonumber
&\qquad\ge
\mathbb P \bigl(\mathcal E^t_{t_n} \cap \mathcal E^s_{t_n} \cap \mathcal E^g_{t_n} \mid \mathcal F_{t_n} , \pi_{t_n} \bigr) 
- \mathbb  P\bigl(\overline{\mathcal E^l_{t_n}} \mid \mathcal F_{t_n} , \pi_{t_n} \bigr)
\\
\label{eq: partial decomposition of the joint probabilities}
&\qquad= \mathbb E \left[ \mathbf 1_{\mathcal E^t_{t_n} \cap \mathcal E^s_{t_n}} \mathbb P(\mathcal E^g_{t_n} \mid \mathcal F_{T_n}) \mid \mathcal F_{t_n}, \pi_{t_n} \right] - \mathbb P\bigl(\overline{\mathcal E^l_{t_n}} \mid \mathcal F_{t_n} , \pi_{t_n} \bigr)
\end{align}
Given that $T_n > t_n \ge \max \{K_\xi, K_N\}$, we obtain from \autoref{thm: from a_k margin to delta} and our specific choice of $M$ that
\begin{align*}
\mathbb P(\mathcal E^g_{t_n} \mid \mathcal F_{T_n}) 
&\ge 1 - \sum_{(s,a) \in \mathcal B_{\pi_{t_n}}} \mathbb P \bigl( \overline{\mathcal E^g_{T_n}(s,a)} \mid \mathcal F_{T_n} \bigr)
\\
&\ge 1 - |\sset \times \aset| \frac{8 c_\xi}{c^2 \rho_\alpha^2} \sum_{j = 0}^{\infty} \frac{2^j}{(M+ 2^j)^2}
= p_g \ge \frac12 .
\end{align*}
Further, given that $t_n \ge \max \{K_\xi, K_l\}$, we obtain from \autoref{thm: event noise remains small} that
\begin{align*}
\mathbb P \bigl( \overline{\mathcal E^l_{t_n}} \mid \mathcal F_{t_n} , \pi_{t_n} \bigr)
&\le
\sum_{(s,a) \in \mathcal B_{\pi_{t_n}}} \mathbb P \bigl(\overline{\mathcal E^l_{t_n}(s,a)} \mid \mathcal F_{t_n}, \pi_{t_n} \bigr)
\\
&\le
|\mathcal S\times\mathcal A|
\frac{64 c_\xi}{\delta^2}\sum_{i={t_n}}^\infty \alpha_i^2
=
p_{\overline l}
\le
\frac14 ( p_t p_s)^{|\mathcal S \times \mathcal A| M} .
\end{align*}
Thus we can further bound \eqref{eq: partial decomposition of the joint probabilities} by
\begin{multline}
\label{eq: probability lower bound intermediary}
\mathbb P \big(
\forall \, s \in \sset, \
\forall \, k \in \{t_n, \dots, t_{n+1}\}
:
(s,\pi_{k}(s)) \notin \mathcal B_{\pi_{t_n}}
\mid \mathcal F_{t_n}, \pi_{t_n}
\big)
\\
\ge
p_g \, \mathbb P(\mathcal E^t_{t_n} \cap \mathcal E^s_{t_n} \mid \mathcal F_{t_n}, \pi_{t_n}) - p_{\overline l} .
\end{multline}
We now lower-bound $\mathbb P(\mathcal E^t_{t_n} \cap \mathcal E^s_{t_n} \mid \mathcal F_{t_n}, \pi_{t_n})$ by sequentially conditioning over $\{t_n, \dots, T_n-1\}$.
For this, define
\[
\ell(k) := 1 + \left\lfloor \frac{k-t_n}{M} \right\rfloor \in \{1,\dots,L\} ,
\]
so that \(k\) lies in the \(M\)-step seed block of pair \((s_{\ell(k)},a_{\ell(k)})\).
We decompose events \(\mathcal E^s_{t_n}\) and \(\mathcal E^t_{t_n}\) into one-step events as follows. For $k \in \{t_n,\dots,T_n-1\}$, let
\[
\mathcal E^s_{t_n}(k)
:=
\left\{
u_{k}(s_{\ell(k)},\pi_{t_n}(s_{\ell(k)}))=1,\ 
v_{k}(s_{\ell(k)},\pi_{t_n}(s_{\ell(k)})) \ge -\dfrac{d_{\min}}{4}
\right\}
\]
if $q_{\pi_{t_n}} \big(s_{\ell(k)},\pi_{t_n}(s_{\ell(k)}) \big) - q_{b_{\ell(k)}} \big(s_{\ell(k)},\pi_{t_n}(s_{\ell(k)}) \big) \ge d_{\min}/2$, 
and
\[
\mathcal E^s_{t_n}(k)
:=
\left\{
u_{k}(s_{\ell(k)},a_{\ell(k)})=1,\ 
v_{k}(s_{\ell(k)},a_{\ell(k)}) \le \dfrac{d_{\min}}{4}
\right\}
\]
otherwise.
Similarly, define
\[
\mathcal E^t_{t_n}(k)
:=
\left\{
\forall \, s \in \mathcal S : 
(s,\pi_k(s)) \notin \mathcal B_{\pi_{t_n}}
\right\}
\]
if $k \le t_{n+1}$, and
\[
\mathcal E^t_{t_n}(k)
:= \Omega
\]
otherwise,
where $\Omega$ is the certain event.
These one-step events yield the decompositions
\begin{align*}  
\mathcal E^s_{t_n} = \bigcap_{k=t_n}^{T_n-1}\mathcal E^s_{t_n}(k),
\qquad
\mathcal E^t_{t_n} = \bigcap_{k=t_n}^{T_n-1}\mathcal E^t_{t_n}(k).
\end{align*}

Fix a time $k \in \{t_n, \dots, T_n\}$ and consider the joint event
\begin{align}
\label{eq: joint intermediate event}
\bigcap_{i=t_n}^{k-1} \bigl( \mathcal E^t_{t_n}(i)\cap \mathcal E^s_{t_n}(i) \bigr).
\end{align}
We first bound the tie-breaking event. If $k \le t_{n+1}$, then all gaps remain non-negative up to time $k$. A pair $(s,a) \in \mathcal B_{\pi_{t_n}}$ can thus be greedy only if its gap is zero. In that case, $\pi_{t_n}(s)$ is also greedy. So, on \eqref{eq: joint intermediate event}, $\mathcal E^t_{t_n}(k)$ holds whenever the policy $\pi_k$ selects $\pi_{t_n}(s)$ in each such state. By condition \eqref{eq:greedy-pol-tie-breaking} and independence across states,
\[
\mathbb P\!\left(\mathcal E^t_{t_n}(k)\mid \mathcal F_k,\pi_{t_n}\right)
\ge \beta^{|\mathcal S|} .
\]
If $k>t_{n+1}$, then $\mathcal E^t_{t_n}(k)=\Omega$, so the same bound holds. Thus,
\[
\mathbb P\!\left(\mathcal E^t_{t_n}(k)\mid \mathcal F_k,\pi_{t_n}\right) \ge p_t.
\]
For the seed event, the strict exploring starts condition in \autoref{asp: standing}.(3), together with
\eqref{eq:noise-lower}--\eqref{eq:noise-upper}, yields
\[
\mathbb P\!\left(\mathcal E^s_{t_n}(k)\mid \mathcal F_k, \pi_{t_n}\right)
\ge \sigma_{\min}\frac{d_{\min}^2}{d_{\min}^2+16c_v}
= p_s.
\]
Combining both bounds and using the tower property, we obtain
\begin{align*}
\mathbb P\!\left(\mathcal E^t_{t_n}(k) \cap \mathcal E^s_{t_n}(k)\mid \mathcal F_k, \pi_{t_n}\right)
\ge
p_s\,\mathbb P\!\left(\mathcal E^t_{t_n}(k) \mid \mathcal F_k,\pi_{t_n}\right)
\ge p_t p_s.
\end{align*}
Applying this recursively yields
\begin{align*}
&\mathbb P \! \left(
\bigcap_{i=t_n}^{k} \bigl( \mathcal E^t_{t_n}(i) \cap \mathcal E^s_{t_n}(i) \bigr)
\ \middle| \ \mathcal F_{t_n}, \pi_{t_n}
\right)
% \\
% &\qquad\ge
\ge
p_t p_s\,
\mathbb P\!\left(
\bigcap_{i=t_n}^{k-1} \bigl( \mathcal E^t_{t_n}(i)\cap \mathcal E^s_{t_n}(i) \bigr)
\ \middle|\ \mathcal F_{t_n}, \pi_{t_n}
\right) \!.
\end{align*}
Iterating from $t_n$ to $T_n-1=t_n+LM-1$ gives
\[
\mathbb P(\mathcal E^t_{t_n} \cap \mathcal E^s_{t_n} \mid \mathcal F_{t_n}, \pi_{t_n})
\ge (p_t p_s)^{LM}
\ge (p_t p_s)^{|\mathcal S\times\mathcal A|\,M}.
\]

Altogether, define
\begin{align*}
% \label{eq: def lower bound probability} 
p
:=
\frac14
(p_t \, p_s)^{|\mathcal S \times \mathcal A| \, M}
% \\
% \nonumber
=
\frac14
\left(
\beta^{|\mathcal S|}
\sigma_{\min}\frac{d_{\min}^2}{d_{\min}^2+16 c_v}
\right)^{|\mathcal S\times\mathcal A| \, M}
>0  .
\end{align*}
Then for every MC-O-PI solution there exists an almost surely finite $K \ge \max\{K_\xi, K_N, K_l\}$ such that for every transition time $t_n$ satisfying $K \le t_n < \infty$,
\begin{align*}
\mathbb P \big(
\forall \, s \in \sset, \
\forall \, k \in \{t_n, \dots, t_{n+1}\}
:
(s,\pi_{k}(s)) \notin \mathcal B_{\pi_{t_n}}
\mid \mathcal F_{t_n}, \pi_{t_n}
\big)
\ge
p_g (p_t p_s)^{|\mathcal S\times\mathcal A|\,M} - p_{\overline l}
\ge
p .
\end{align*}
Since this bound is uniform in $\pi_{t_n}$, \eqref{eq: porb better action value} follows.
% \end{proof}
\hfill\BlackBox

Altogether, the proof of \autoref{thm: fixed prob that bad actions dont become greedy} relies on three successive events.
These events ensure that stochastic perturbations cannot consistently
undermine the policy improvement driven by the mean-field dynamics
by forcing ``bad'' actions to become greedy.

% use advantage instead of gap?

The lock-in phenomenon is a direct application of Theorem~2.1 in \citep{Borkar2002OnApproximation}:
if the gap variable for a ``bad'' pair reaches a fixed positive margin infinitely often, 
it must eventually remain strictly positive.
However, since decreasing learning rates cause these gaps to be increasingly smaller at initialization,
we introduced a growth event to ensure they still reach this fixed margin.
This event extends the lock-in phenomenon to hold under a milder condition: gaps only need to reach a decreasing margin of size $O(\alpha_k)$ infinitely often.
Still, it requires an additional comparability condition on the learning rates so that, with a fixed probability, once a gap hits the $O(\alpha_k)$ margin, the remaining learning rates are sufficiently large to bring it up to the original fixed margin.
Finally, to guarantee gaps reach the decreasing $O(\alpha_k)$ margin within a fixed horizon, we used a seed event.
The length of this horizon determines the probability of the growth event.
By fixing a sufficiently long horizon, we can drive the probability of the growth event arbitrarily close to one, regardless of the unknown parameters in \eqref{eq: prob from a_k reach delta}.
The only trade-off for this extended horizon is a reduced probability of the seed event.
Yet, since that probability is fixed once the horizon is fixed, the overall probability of improvement remains fixed.

% \textcolor{red}{this extends the lock-in strategy of mcopi}

It is crucial for the seed event that updates use the initial-visit method to guarantee that all gaps initially increase.
% only one action be updated in any state at every step so 
% and that potential correlation in the episode noise does not break \eqref{eq: increase in d^b case B}.
% The structure of the stochastic perturbations $w_k$ is crucial for the seed event.
% Since the stochastic perturbations caused by episode simulations can be adversarially correlated, meaning they could 
% (which is not really an issue because uniform updates require initial-visit anyway in practice)
% to exactly cancel out the update in all but one state-action pair at every step (initial-visit).
However, uniform update distributions are irrelevant for this part of the proof because the strict exploring starts guarantee that we can always construct the seed event.
On the contrary, the growth and lock-in events essentially enforce that the MC-O-PI solution behaves well because it remains close to the mean-field solution.
Thus they rely strongly on the uniformity of the update distributions that ensures the mean-field solution behaves well in the first place.
The update function is, however, irrelevant.
All that is necessary is that the stochastic perturbation $w_k$ is unbiased and that its variance is eventually bounded.

% Ultimately, the proof thus goes through the same
% under 
% initial-visit updates (allowing us to define the seed event),
% Robbins--Monro learning rates $(\alpha_k)$ that satisfy the comparability property,
% lower- and upper-bounded functions $(\mu_k)$ that are uniform over the actions in each state but do not necessarily integrate to one,
% and unbiased perturbations $(w_k)$ with eventually bounded variance.

%%%%%%%%%%%%%%%%%%%%%%%%%%%%%%%%%%%%%%%%%
%%%%%%%%%%%%%%%%%%%%%%%%%%%%%%%%%%%%%%%%%
\subsubsection{Global convergence of MC-O-PI}

Now that we have established that stochastic perturbations cannot consistently cancel the policy improvement of the mean-field, global convergence of MC-O-PI to $\qopt$ is immediate.

\begin{proposition}
\label{thm: convergence mcopi uniform}
    Solutions of initial-visit MC-O-PI, as defined in \eqref{eq: mcopi sa},
    converge almost surely to the optimal action values $\qopt$ under \autoref{asp: standing},
    if
    the learning rates satisfy the comparability assumption (\autoref{asp: additional conditions on learning rate})
    % the update functions are initial-visit,
    and the sampling distributions of initial state-action pairs are uniform over all actions in each state (\autoref{asp: unbiased over policies}).
\end{proposition}

\begin{proof}
Let $(q_k, \pi_k)_{k \ge 0}$ be an MC-O-PI solution with transition times $(t_n)_{n \ge 0}$.
By \autoref{thm: fixed prob that bad actions dont become greedy}, there exists an almost surely finite $K$ such that, with a fixed probability $p$,
\begin{align}
\label{eq: event no bad gaps become greedy}
(s,\pi_k(s)) \notin \mathcal B_{\pi_{t_n}}
\qquad
\forall \, K \le t_n < \infty , \
\forall \, t_n \le k \le t_{n+1} , \
\forall \, s \in \mathcal S .
\end{align}
On this event, if $t_{n+1}<\infty$, then
\begin{align}
\label{eq: goal of proof stoch improvement proof}
q_{\pi_{t_n}} ( s, \pi_{t_{n+1}}(s) )
\ge
q_{\pi_{t_n}} ( s,\pi_{t_n}(s) )
\qquad \forall\, s \in \mathcal S ,
\end{align}
so $\pi_{t_{n+1}}$ improves upon $\pi_{t_n}$ by \autoref{thm: strict PIT}.
If instead $t_{n+1}=\infty$, then $\pi_{t_n}=\pi_*$ by \autoref{thm: suboptimal blocks finite}.
% and $q_k \to q_{\pi_*}$ by \cite[Proposition~4.1, Example~4.3]{Bertsekas1996Neuro-DynamicProgramming}.
Hence, from any $t_n \ge K$, the policy either improves or is already optimal and remains so with probability at least $p$.

Since a sequence of deterministic policies can improve at most $|\mathcal P|-1$ times before reaching optimality, 
the probability of consecutive improvements until the policy becomes optimal and then remains so is at least $p^{|\mathcal P|}$.
Equivalently, for every $n$ with $K\le t_n<\infty$,
\begin{align}
\label{eq:prob-lock-within-N-transitions}
\mathbb P\!\left(
t_{n+|\mathcal P|} < \infty
\mid
\mathcal F_{t_n}
\right)
\le 1 - p^{|\mathcal P|}.
\end{align}

Fix $n \in \mathbb Z_{\ge 0}$ such that $K\le t_n<\infty$. 
Iterating \eqref{eq:prob-lock-within-N-transitions} over blocks of length $|\mathcal P|$ gives, 
for every $i \ge 1$,
\begin{align*}
\mathbb P\big(
t_{n+i|\mathcal P|} < \infty
\mid
\mathcal F_{t_n}
\big)
&=
\mathbb E\!\left[
\mathbf 1_{\{t_{n+(i-1)|\mathcal P|}<\infty\}}
\,
\mathbb P\big(
t_{n+i |\mathcal P|} < \infty
\mid
\mathcal F_{t_{n+(i-1)|\mathcal P|}}
\big)
\ \middle| \
\mathcal F_{t_{n}}
\right]
\\
&\le
(1-p^{|\mathcal P|})\,
\mathbb P (
t_{n+(i-1) |\mathcal P|} < \infty
\mid
\mathcal F_{t_{n}}
)
\\
&\le
\big(1-p^{|\mathcal P|}\big)^i .
\end{align*}
Since $p^{|\mathcal P|} > 0$, the right-hand side converges to zero as $i$ grows to infinity.
Therefore,
\[
\mathbb{P}\!\left(
\bigcap_{i=1}^\infty
\big\{ t_{n+i|\mathcal P|}<\infty \big\}
\;\middle|\;
\mathcal F_{t_n}
\right)
=0.
\]
Thus, almost surely, there exists a finite $i_*$ such that $t_{n+i_*|\mathcal P|} = \infty$.
As a result, there is a transition index $n_* < n + i_* |\mathcal P|$ such that $t_{n_*} < \infty$ and $t_{n_* +1} = \infty$.
It follows that $\pi_{t_{n_*}} = \pi_*$ by \autoref{thm: suboptimal blocks finite} and $q_k \to q_{\pi_*}$ by Proposition~4.1 in \citep{Bertsekas1996Neuro-DynamicProgramming}.
\end{proof}

The key insight of this result is to analyse the sequence of greedy policies rather than the value estimates. This exposed a simple and practical condition for monotone policy improvement of the mean-field dynamics, making the proof a straightforward application of the Policy Improvement Theorem.
The main technical work was to show that the stochastic perturbations cannot persistently obstruct this improvement. 
This policy-level argument avoids some of the delicate value-based arguments required in earlier analyse, such as the extensions from $\gamma<1$ to $\gamma=1$.
% \newpage
\section{Discussion}
\label{sec: discussion}

\autoref{thm: convergence mcopi uniform} showed that initial-visit MC-O-PI converges almost surely to $\qopt$ 
% for a contractive MDP 
when updates are uniform over the actions in each state.
% or, equivalently, when they are uniform over the deterministic policies.
% 
% provided
% that the learning rates satisfy the Robbins--Monro conditions and the comparability property,
% % the update functions are initial-visit,
% that the sampling distributions satisfy strict exploring starts,
% and that the resulting update distributions are uniform over all actions in each state.
% 
This section explores the implications of this result.
We first explain that the uniformity condition effectively requires using initial-visit updates, which does not necessarily slow down convergence compared to first-visit updates.
We then describe how to generalise this result to optimistic policy-iteration algorithms with temporal-difference updates, or when the estimate $q_k$ is stored approximately at every step.
Finally, we argue that some conditions can likely be relaxed because they are not essential for the underlying improvement process.

% as the condition for uniform updates forces the use of initial-visit method
% Even if initial-visits is less data-efficient than first-visits, this does not mean that solutions converge slower to the optimal policy.

%%%%%%%%%%%%%%%%%%%%%%%%%%%%%%%%%%%%%%%%%%%%%%%%%%%%%%%%%%%%%%%%%%%
%%%%%%%%%%%%%%%%%%%%%%%%%%%%%%%%%%%%%%%%%%%%%%%%%%%%%%%%%%%%%%%%%%%
\subsection{Uniform updates effectively force initial-visits}

When the underlying MDP is unknown, using first- or every-visit updates causes the update distribution $\mu_k$ to depend on the unknown transition dynamics induced by policy $\pi_k$.
Therefore, requiring uniform update distributions effectively forces the use of initial-visit updates,
% because we cannot determine how the initial sampling distribution $\sigma_k$ will propagate through the MDP under policy $\pi_k$. 
because by updating only the initial state-action pair of each episode, we ensure $\mu_k$ equals $\sigma_k$ and thus retain control.

This control comes at the cost of data efficiency: the initial-visit method discards within-episode transitions that the first-visit method would reuse. However, more updates per episode do not automatically translate into faster policy convergence.
The extra updates induce a bias in the update distribution that can generate sliding modes in the action-value estimates $(q_k)_{k \ge 0}$ and thereby slow down the convergence of the policy sequence $(\pi_k)_{k \ge 0}$. Ultimately, it is the convergence speed of this policy sequence that truly matters for identifying the optimal behaviour of the MDP.

% describe MDP
To illustrate this, consider the MDP in \autoref{fig: experiment smallest MDP for iv vs fv} with a discount factor $\discount = 0.5$. Action \textit{left} yields a reward of one and \textit{right} yields a reward of zero.
% The best policy $\best{\pol}$ follows the \textit{left} action, whereas the worst policy $\worst{\pol}$ follows the \textit{right} action.
% describe learning algorithm
Under a uniform sampling distribution $\sigma_k$, the intial-visit method updates exactly one state-action pair per episode, whereas the first-visit method updates on average $1.5$ pairs.
With a learning rate $\alpha_k = 1/(k+100)$, the first-visit method also benefits from larger effective steps because it updates state-action pairs more frequently.
% Despite this numerical advantage, experiments show that IV can convergence to $\pi^*$ sooner.

Despite this apparent advantage, the initial-visit method is more stable than first-visit when MC-O-PI is initialised above $\qopt$.
Namely, initial-visit solutions follow straight lines that always point towards better policies.
In contrast, first-visit solutions exhibit a sliding mode that causes high-frequency oscillations in the policy sequence and traps $q_k$ near the boundary for longer.
As a result, first-visit can take longer to settle on the optimal policy.
For instance, when averaging over 1000 solutions starting in $q'$, first-visit ones require on average 43 updates to settle on $\pi_*$, whereas initial-visit ones always remain in $\region(\best{\pi})$.
When starting in $q$, the solutions take 68 and 70 updates to settle on $\pi_*$.

Overall, this experiment illustrates that the initial-visit method can be practical in real-world applications. 
By avoiding sliding modes, it sometimes converges to the optimal policy $\best{\pi}$ as fast as, or faster than, the first-visit method despite fewer updates and smaller effective learning rates.
Besides, the initial-visit method is more reliable for early stopping. Namely, if training stops early, the last greedy policy of an initial-visit solution is typically the best policy encountered so far, whereas the first-visit solution may be oscillating across a decision boundary, making it unclear which policy is best.

% Namely, since the policy sequence transitions almost smoothly instead of oscillating, terminating the algorithm prematurely yields a policy that reliably reflects the best policy found so far. In contrast, stopping a first-visit trajectory early risks catching it across a decision boundary, leaving it unclear which policy is actually best.

% since the initial-visit method experiences no sliding modes, 
% it can converge as fast as, or faster than, the first-visit method to $\best{\pi}$, even with fewer updates and smaller effective learning rates. Furthermore, the absence of sliding mode makes
% , the current policy is a more reliable “best-so-far” policy. So if training stops early, the IV iterate typically reflects the best policy reached so far, while FV may still oscillate across the decision boundary, making it unclear which policy is best.

% this problem is not conthrived
% - happens when overestimate the action values (a priori dont know them)
% - dont know how strong the greedy bias with FV is, this is the region with potential sliding modes

% if we use learning rate $1/n(s,a)$
% actually helps FV because learning rate decreases the greedy bias

\begin{figure}[t!]
    \centering
    \begin{subfigure}[b]{\textwidth}
        \centering
        % \begin{tikzpicture}[
        %     scale=1, auto, node distance=2cm, 
        %     every edge/.style={draw, <-, font=\small},
        %     state/.style={circle, fill=mylightgrey, draw=black, text=black}
        % ]
        %     % Nodes
        %     \node[state] (s1) at (0,0) {$s$};
        %     % Edges
        %     \path[->]
        %         (s1) edge[out=160, in=200, loop, looseness=8, draw=black] node[pos=0.5, left, text=black] {\textcolor{mygrey}{$\best{\pi}(s) :=$} 1} (s1)
        %         (s1) edge[out=20, in=340, loop, looseness=8, draw=black] node[pos=0.5, right, text=black] {0 \textcolor{mygrey}{$=: \worst{\pi}(s)$}} (s1) ;
        %         % (s1) edge[out=145, in=215, loop, looseness=7] 
        %         %     node[pos=0.5, right] {$\best{a}$} 
        %         %     node[pos=0.5, left] {1} 
        %         % (s1)
        %         % (s1) edge[out=35, in=325, loop, looseness=7] 
        %         %     node[pos=0.5, left] {$\worst{a}$} 
        %         %     node[pos=0.5, right] {0} 
        %         % (s1);
        % \end{tikzpicture}
        \includegraphics[width=0.35\linewidth]{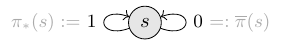}
        \caption{MDP}
        \label{fig: experiment smallest MDP for iv vs fv}
    \end{subfigure}%
    \\
    \vspace{0.5cm}
    % --- Subfigure 1: Streamlines ---
    \begin{subfigure}[b]{0.48\textwidth}
        \includegraphics[width=0.8\linewidth]{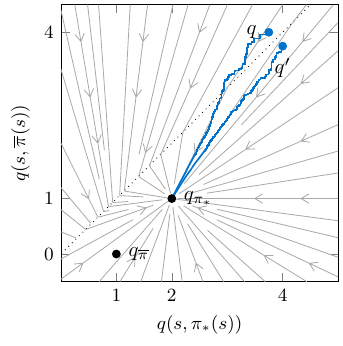}
        \caption{Initial-visit solutions}
        \label{fig: iv versus fv simulations IV}
    \end{subfigure}
    \hfill
    % --- Subfigure 2: Placeholder or Second Simulation ---
    \begin{subfigure}[b]{0.48\textwidth}
        \includegraphics[width=0.8\linewidth]{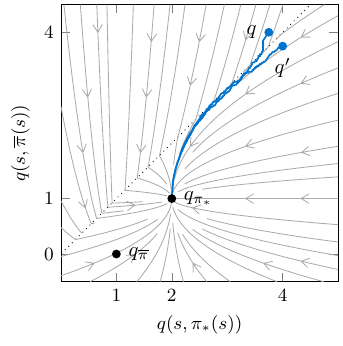}
        \caption{First-visit solutions}
        \label{fig: iv versus fv simulations FV}
    \end{subfigure}
    \caption{By avoiding sliding modes, initial-visit solutions sometimes converge to the optimal policy $\best{\pi}$ faster than first-visit solutions despite fewer updates and smaller effective learning rates. The blue lines are MC-O-PI solutions starting in $q$ or $q'$, whereas the grey lines indicate the mean-field flow.}
    \label{fig: iv versus fv simulations}
\end{figure}

\subsection{Mean-field analysis with temporal difference updates}
\label{sec: discussion temporal difference}

The growth and lock-in events only require unbiased noise with eventually bounded variance, and the seed event is a simple finite-horizon process to guarantee that solutions can overcome the decreasing margin.
These events should therefore also make it possible to derive global convergence of optimistic policy-iteration algorithms with temporal-difference updates from their mean-field dynamics.
The additional difficulty would be to understand how the position of the estimate $q$ and the depth $n$ of the episode simulations affect the target $T_{\pi}^n q$, where $\pi \in P(q)$ and $T_\pi$ is the Bellman operator associated with policy $\pi$.

%%%%%%%%%%%%%%%%%%%%%%%%%%%%%%%%%%%%%%%%%%%%%%%%%%%%%%%%%%%%%%%%%%%
%%%%%%%%%%%%%%%%%%%%%%%%%%%%%%%%%%%%%%%%%%%%%%%%%%%%%%%%%%%%%%%%%%%
\subsection{MC-O-PI with function approximation}
\label{sec: discussion function approximation}

\autoref{thm: convergence mcopi uniform} applies specifically to \textit{tabular} MC-O-PI, where the estimate $q_k$ can be stored exactly.
In practical applications, however, the estimate is approximated, for instance using a neural network.
In this setting, the error between the observed return and the current estimate is then used to update the parameters of the approximation.
% for instance using gradient descent.

With a sufficiently rich approximation method,
the approximation error should only contribute a finite perturbation to the MC-O-PI difference inclusion.
This suggests that \autoref{thm: convergence mcopi uniform}
can be generalised with function approximation using the results in \citep[Section~4.3.3]{Bertsekas1996Neuro-DynamicProgramming} or \citep[Section~5.2]{Kushner2003StochasticApplications}.
%Furthermore, integrating a normalisation trick as in \autoref{sec: normalisation trick} into Monte Carlo Tree Search might accelerate AlphaZero's or MuZero's learning process by mitigating slow sliding modes.

% Finally, \citet{Wagner2013OptimisticResult} shows that policy gradient methods
% % which underlie recent breakthroughs in natural language processing and robotics \citep{Schulman2017ProximalAlgorithms}, 
% can be interpreted as optimistic policy iteration algorithms with function approximation. Extending our results to this setting might
% thus improve the efficiency and performance of these essential methods.

Finally, recent breakthroughs in natural language processing and robotics~\citep{Schulman2017ProximalAlgorithms} rely on policy gradient methods.
\citet{Wagner2013OptimisticResult} shows that some of these methods can be interpreted as optimistic policy iteration algorithms with function approximation.
Extending our results to this setting might thus improve the efficiency and performance of these essential methods.

\subsection{Relaxing conditions}
\label{sec: relaxing conditions}

% - assumption that limit set reduces to a point => this is assumption on solution, not on algorithm

% relax
% - policy samping => update only policy in state that changed, then dont even need event $\mathcal E^t_k$
% - essentially just need comparison prop of leanring rates to ensure reach fixed distance into a region, if can show other way, then just need Borkar lock in, i.e. \autoref{thm: event noise remains small}

% Technically, 
The proof of \autoref{thm: convergence mcopi uniform}
requires a uniform lower-bound on the distribution of the greedy policy at each step, strict exploring starts and non-increasing, comparable learning rates.
However, these conditions are artefacts of the framework of the proof rather than necessary conditions for the underlying policy improvement process.

For instance, the global seed event updates only one state-action pair per step. 
% Yet the greedy policy can be updated for all states at every step.
Yet since the greedy policy is updated for all states at every step, it can change in any state with ties.
Sampling the greedy policy, as described in \eqref{eq:greedy-pol-tie-breaking}, 
therefore prevents
adversarial tie-breaking from always worsening the policy by choosing $\pi_{k+1}(s)=a$
whenever
% if after the $k^{\text{th}}$ update,
$q_{k+1}(s,\pi_k(s)) = q_{k+1}(s,a)$ but $q_{\pi_k}(s,\pi_k(s)) > q_{\pi_k}(s,a)$.
% the gap of a pair in $\mathcal B_\pi$ is zero after an update,
% , and thereby invalidate \autoref{thm: fixed prob that bad actions dont become greedy}.
% (see the discussion from equation \eqref{eq: joint intermediate event} onwards)
Such ties are rare in practice;
% except if the algorithm is initialised in one.
moreover, updating the greedy policy for all states at every step is inefficient.
A more practical method under initial-visit updates is to only recompute the policy in the state that was updated at step $k$.
%by comparing $q_{k+1}(s, \pi_{k}(s))$ with $q_{k+1}(s, a)$ where $u_{k}(s,a) = 1$.
This method simplifies the proof of \autoref{thm: convergence mcopi uniform} as the global tie-breaking event \eqref{eq: def global tie event} is no longer necessary,
given that ties are broken once, right after the transition, 
and remain unchanged until the corresponding seed event makes the policy-worsening action non-greedy.
% until the pair is updated by the corresponding seed event,
% which can only increase the gap and hence break the tie in favour of the currently greedy action.

Further, the strict exploring starts, and the non-increasing and comparability conditions on the learning rates ensure that the stochastic perturbations cannot persistently offset the mean-field dynamics.
Specifically, they guarantee through the seed and growth events
% \eqref{eq: def global seed event}--\eqref{eq: def global grow event}
that
for every policy $\pi \in \mathcal P$,
the stochastic perturbations cannot force the sequence
$(q_k(s,\pi(s)) - q_k(s,a))_{k \ge 0}$ to converge to zero 
on the subsequence of times $\{k \ge 0 : \pi_k = \pi\}$
if $q_{\pi}(s,\pi(s)) > q_{\pi}(s,a)$.
However, given the Robbins--Monro learning rates and the martingale difference noise sequence,
the accumulated stochastic perturbations form a converging martingale sequence, and hence $\lim_{K \to \infty} \sum_{k \ge K} \alpha_k w_k = 0$.
As a result, we expect that there exists another way to establish that the perturbations do not contain enough energy to consistently counteract the mean-field dynamics.
% , likely using an equivalent of the Robbins-Siegmund Theorem.

% If can prove another way, then can use the lock-in phenomenon and the combined stability-ODE method described in Section \ref{sec: limitations of existing methods}.

% %\textcolor{red}{is all this necessary? Could just skip to stating the minimal requirements}
Since the monotone improvement that drives convergence when updates are uniform does not hinge on these conditions, we expect that they can be relaxed.
Nevertheless, some constraints on the update distributions will always remain necessary as they effectively scale the learning rates in the MC-O-PI difference inclusion \eqref{eq: mcopi not yet sri}.
% leading to different learning rates for different state-action pairs.
Namely, convergence to a fixed point requires Robbins–Monro–like conditions on the products $(\lrate_k \update_k)$, not on $(\lrate_k)$ alone.

% To make sure that the effective learning rates for each pair are large enough to avoid premature convergence due to Zeno-like behaviour while gradually decreasing to zero so that noise cannot prevent convergence, it is that the Robbins--Monro conditions will have to be satisfied on the product $(\lrate_k \update_k)$ rather than on the learning rates alone.

To see why the Robbins--Monro conditions must be adjusted, fix a policy $\pi$ and consider the system
\begin{align}
\label{eq: single system}
    q_{k+1}
    &= q_k + \lrate_k \update_k (q_\pi - q_k) + \lrate_k w_k .
\end{align}
Unrolling over $n$ steps yields
\begin{align*}
    q_{k+n}
    = \left( \prod_{i=k}^{k+n-1} (\mathsf{I} -\lrate_i \update_i) \right) q_k 
    + \left( \mathsf{I} - \prod_{i=k}^{k+n-1} (\mathsf{I} -\lrate_i \update_i) \right) q_\pi
    % \\
    + \sum_{j=k}^{k+n-1} \prod_{i=j+1}^{k+n-1} (\mathsf{I} -\lrate_i \update_i) \lrate_j w_j .
\end{align*}
For $q_{k+n}$ to converge to $q_\pi$ as $n$ grows to $\infty$, three conditions must be satisfied for every state $s \in \sset$ and action $a \in \aset(s)$.
First,
\begin{align}
\label{eq: condition infinite visits}
    \sum_{k=0}^{\infty} \update_k(s,a) = \infty ,
    % \qquad \forall \, s \in \sset, \forall \, a \in \aset(s) .
\end{align}
so that the state-action pair is updated infinitely often.
% Therefore, for every state-action pair $(s,a)$, the update distributions $(\update_k)$ must satisfy
Second, 
\begin{align}
    \label{eq: condition infinite effective learning rate}
    \sum_{k=0}^{\infty} \lrate_k \update_k(s,a) = \infty ,
    % \qquad \forall \, s \in \sset, \forall \, a \in \aset(s) .
\end{align}
so that the product $\prod_{i=k}^{k+n-1} (1 - \lrate_i \update_i(s,a))$ decays to zero for increasing $n$.
Finally,
\begin{align}
    \label{eq: summable learning rates outside def}
    \sum_{k=0}^\infty \lrate_k^2 < \infty ,
\end{align}
to ensure the accumulated stochastic perturbations vanish.

% If we imposed $\sum_k \lrate_k^2 \update_k^2 < \infty$ instead, we could choose $\lrate_k = \update_k(s,a) = 1/\sqrt{k}$ for each state-action pair, still satisfy conditions \eqref{eq: condition infinite visits} and \eqref{eq: condition infinite effective learning rate}. However, the noise term in \eqref{eq: single system unrolled} would become
% \begin{align}
%     \sum_{j=k}^{k+n-1} \prod_{i=j+1}^{k+n-1} (1-\frac{1}{i}) \frac{1}{\sqrt{j}} \xi_j
%     &= \sum_{j=k}^{k+n-1} \frac{j}{k+n+1} \frac{1}{\sqrt{j}} \xi_j
%     \\
%     &= \frac{1}{k+n+1} \sum_{j=k}^{k+n-1} \sqrt{j} \xi_j
%     \\
%     &\eqqcolon N_{k,n}
% \end{align}
% We see that
% \begin{align}
%     \E[N_{k,n}^2] = \E \Bigg[ \Big( \frac{1}{k+n+1} \sum_{j=k}^{k+n-1} \sqrt{j} \xi_j \Big)^2 \Bigg]
% \end{align}
% Since the cross terms vanish in expectation, 
% and as a result, $q_k$ would not converge to $q_\pi$
% 
% since $\sum_k \lrate_k^2 = \infty$, the noise would not be almost surely bounded. In other words, the noise alone would be able to pull solution arbitrarily far from $q_\pi$ infinitely often, thereby preventing convergence.
% 
Note that \eqref{eq: condition infinite visits} does not automatically imply \eqref{eq: condition infinite effective learning rate}.
For instance, under \eqref{eq: condition infinite visits} and the Robbins--Monro conditions on the learning rates $(\lrate_k)$, we could choose $\lrate_k = \update_k(s,a) = 1/k$. System \eqref{eq: single system} would then converge prematurely due to Zeno-like behaviour because the products $\prod_{i=k}^{k+n-1} (1-\lrate_i \update_i(s,a))$ would not decay to 0.

Altogether, the minimal requirements on the learning rates and the update distributions are
% for every state $s \in \sset$ and action $a \in \aset(s)$
\begin{align}
    \label{eq: xtn robbins-monro first}
    &\sum_{k =0}^\infty \lrate_k \update_k(s,a) = \infty 
    \qquad \forall \, s \in \sset, \forall \, a \in \aset(s) ,
    \\
    \label{eq: xtn robbins-monro second}
    &\sum_{k =0}^\infty \lrate_k^2 < \infty ,
\end{align}
because \eqref{eq: xtn robbins-monro second} implies that, eventually, $\alpha_k \mu_k < \mu_k$, and \eqref{eq: xtn robbins-monro first} then lower bounds \eqref{eq: condition infinite visits}.
We expect the results from Section~\ref{sec: results} to extend under these milder conditions.

\section{Conclusion}

%\textcolor{red}{mention that our result make more sense / uniform over all state-action pairs automatically forces more frequent updates to states with more actions, which does not really make sense}

 % 1. Finding = what did we find
We established that the mean-field dynamics of Monte Carlo optimistic policy iteration drive trajectories toward monotonically improving policies when updates are uniform over all actions in each state. 
We then proved that the stochastic perturbations cannot systematically counteract this improvement,
and consequently that MC-O-PI converges to the optimal action values when updates are uniform.

% Consequence of the result
This result extends existing convergence guarantees to more practical settings. Unlike previous proofs that required uniform updates across all state-action pairs, our analysis allows different states to be updated at different frequencies.
This relaxation makes the algorithm applicable to large-scale problems where uniform sampling over the joint state-action space is not feasible.
We now know that global convergence remains guaranteed as long as it is possible to sample the action space of each state uniformly.
% Instead, learning can prioritise different states at different times by sampling the initial states of episodes accordingly.
% As long as every state is sampled infinitely often and all actions of each state are updated equally frequently, MC-O-PI eventually converges to the optimal policy and action values. 

% consequence of the method
We extended the lock-in strategy of the combined stability-ODE method to accommodate the discontinuous dynamics of MC-O-PI.
This mean-field approach provided a more intuitive explanation for the convergence mechanism than previous proofs that directly tackled the stochastic process.
More importantly, it opens up new possibilities, beyond contraction-based analysis, to study the convergence of optimistic policy iteration algorithms.

% perspectives
We see three principal directions for future work:
(1) remove the conditions of \autoref{thm: convergence mcopi uniform} that are not necessary for the mean-field improvement property,
% (2) find whether the optimal action values remain globally stable for mean-field solutions of MC-O-PI when updates are non-uniform,
(2) apply this new method to study the convergence of optimistic policy iteration algorithms that rely on temporal difference evaluation,
and (3) generalise our result when value estimates are stored approximately.

\clearpage

% Acknowledgements and Disclosure of Funding should go at the end, before appendices and references

\acks{This research was supported by the MathWorks Studentship.}


% \newpage

\appendix

% \newpage
\section*{Appendix}
\label{sec: Appendix}

\paragraph{Proof of \autoref{thm: properties of noise}}

% \begin{proof}
Fix a time $k\ge 0$, a state $s \in \mathcal S$ and an action $a \in \mathcal A(s)$.
Let $\mathcal F_k'$ represent the available information right after selecting the policy $\pi_k$ but before sampling the initial state-action pair, so $\mathcal F_k \subseteq \mathcal F_k'$.

By definition,
\begin{align}
\label{eq: noise def in appendix proof}
w_k(s,a) = u_k(s,a) v_k(s,a) + \big( u_k(s,a)-\mu_k(s,a) \big) \big( q_{\pi_k}(s,a) - q_k(s,a) \big).
\end{align}
Since episode returns are unbiased under initial- or first-visit updates, 
and $u_k(s,a) v_k(s,a) = 0$ on $\{u_k(s,a)=0\}$,
it follows that
\begin{align}
\label{eq: unbiased noise p1}
\mathbb E[ u_k(s,a) v_k(s,a) \mid \mathcal F_k']
=
\mathbb E \big[ 
    \mathbf 1_{\{u_k(s,a)=1\}} \mathbb E[ v_{k}(s,a) \mid \mathcal F_k', u_k(s,a)=1] 
    % +
    % \mathbf 1_{\{u_k(s,a)=0\}} \mathbb E[ v_{k}(s,a) \mid \mathcal F_k', u_k(s,a)=o] 
    \mid \mathcal F_k' \big]
= 0 .
\end{align}
Also, $\mathbb E[u_k(s,a) \mid \mathcal F_k'] = \mu_k(s,a)$, and $\mu_k$ and $q_{\pi_k}-q_k$ are $\mathcal F_k'$-measurable. Hence,
\begin{multline}
\label{eq: unbiased noise p2}
\mathbb E \!\left[
\bigl( u_k(s,a)-\mu_k(s,a) \bigr) \bigl( q_{\pi_k}(s,a)-q_k(s,a) \bigr)
\mid \mathcal F_k'
\right]
\\
=
\big( q_{\pi_k}(s,a)-q_k(s,a) \big) 
\big( \mathbb E\!\left[ u_k(s,a) \mid \mathcal F_k'
\right]
-\mu_k(s,a) \big)
=
0 .
\end{multline}
Summing \eqref{eq: unbiased noise p1} and \eqref{eq: unbiased noise p2} yields $\mathbb E[w_k(s,a) \mid \mathcal F_k'] = 0$.
Applying the tower property then proves \eqref{eq: muw_mds}:
\[
\mathbb E [ w_{k}(s,a) \mid \mathcal F_k ]
= \mathbb E \!\left[ \mathbb E [ w_{k}(s,a) \mid \mathcal F_k'] \mid \mathcal F_k \right]
= 0 .
\]

Finally, since $u_k^2(s,a) \le 1$ and $(u_k(s,a) - \mu_k(s,a))^2 \le 1$,
using $(x+y)^2 \le 2 x^2 + 2 y^2$ on \eqref{eq: noise def in appendix proof} yields
\begin{align*}
w_k^2(s,a) 
&\le 2 v_k^2(s,a) + 2 (q_{\pi_k}(s,a)-q_k(s,a) )^2 
\\
&\le 2 v_k^2(s,a) + 2 ( 2 q_{\pi_k}^2(s,a) + 2 q_k^2(s,a) ) .
\end{align*}
Since all action-values are bounded,
there exists a constant $c_\polset < \infty$ such that $\|q_\pi\|^2 \leq c_\polset$ for every policy $\pi \in \polset$. As a result,
\begin{align*}
\mathbb E \!\big[ w_{k}^2(s,a) \mid \mathcal F_k \big]
&\le
2 c_v + 2 \, \mathbb E \!\left[ ( q_{\pi_k}(s,a)-q_k(s,a) )^2 \mid \mathcal F_k \right]
\\
&\le
2 c_v + 4 c_\polset + 4 \|q_k\|^2
\end{align*}
So \eqref{eq: muw_second_moment} holds for example with $c_w := \max \{ 2 c_v + 4 c_\polset, 4\}$.
% \end{proof}
\hfill \BlackBox

%%%%%%%%%%%%%%%%%%%%%%%%%%%%%%%%%%%%%%%%%%%%%%%%%%%%%%%%%%%%%%
\paragraph{Proof of \autoref{thm: bounded limit sets}}

% \begin{proof}
Let $(q_k)_{k \ge 0}$ and $(\pol_k)_{k \ge 0}$ be solutions of the difference inclusion in \eqref{eq: mcopi sa} with initial condition $q_0$, learning rates $(\lrate_k)_{k \ge 0}$, update distributions $(\update_k)_{k \ge 0}$ and stochastic perturbations $(w_k)_{k \ge 0}$.
    
    By assumption, there exists a finite time $K$ such that
    the sequence $(\pol_k)_{k \ge K}$ only contains policies in $\polset_\infty$.
    % the estimate $q_k$ is updated towards action values $q_{\pi_k}$ where $\pi_k \in \polset_\infty$.
    Thus for all $k \ge K$, each state-action pair of the estimate $q_k$ is updated towards a value $q_{\pol_k}(s,a)$ that satisfies
    \begin{align}
        \min_{\pi \in \polset_\infty} q_\pi(s,a) \leq q_{\pi_k}(s,a) \leq \max_{\pi \in \polset_\infty} q_\pi(s,a) .
    \end{align}
    Therefore, define the sequences $(x_k)_{k \ge K}$ and $(y_k)_{k \ge K}$ as
    \begin{align}
        x_\kpo &= x_k + \lrate_k \big( \update_k(s,a) 
        ( \min_{\pi \in \polset_\infty} q_\pi(s,a) - x_k ) + w_k(s,a) \big) ,
        \\
        y_\kpo &= y_k + \lrate_k \big( \update_k(s,a) ( \max_{\pi \in \polset_\infty} q_\pi(s,a) - y_k ) + w_k(s,a) \big) ,
    \end{align}
    and initialised at $x_K = y_K = q_K(s,a)$. From time $K$ onwards we have
    \begin{align*}
        x_k \leq q_k(s,a) \leq y_k .
    \end{align*}
    Proposition~4.1 and Example~4.3 in \citep{Bertsekas1996Neuro-DynamicProgramming} show that $x_k$ and $y_k$ converge almost surely to $\min_{\pi \in \polset_\infty} q_\pi(s,a)$ and $\max_{\pi \in \polset_\infty} q_\pi(s,a)$, respectively.
    % because the noise satisfies conditions \eqref{eq: muw_mds} and \eqref{eq: muw_second_moment}.
    As a result, we conclude that, almost surely,
    \begin{align*}
        &\liminf_{k \to \infty} q_k(s,a)
        \ge \min_{\pi \in \polset_\infty} q_\pi(s,a)
        % \leq q_k(s,a) 
        \\
        &\limsup_{k \to \infty} q_k(s,a)
        \le \max_{\pi \in \polset_\infty} q_\pi(s,a)
    \end{align*}
    for every state-action pair,
    which proves the first part of the lemma.

    Further,
    given that there are a finite number of deterministic policies, $q_k$ converges almost surely to a bounded set.
    As a result, 
    % that sequence is almost surely bounded, meaning that 
    its norm over all $k \geq 0$ is almost surely finite. So there exists an almost surely finite constant $c_Q$ such that $\sup_{k \geq 0} \|q_k\|_\infty^2 < c_Q$.
% \end{proof}
\hfill \BlackBox

\paragraph{Proof of \autoref{thm: convergence deterministic mcopi}}

% \begin{proof}
Let $(\pi_k)_{k \ge 0}$ be the sequence of policies associated with a solution of \eqref{eq: mcopi sa}
and let $(t_n)_{n \ge 0}$ be the corresponding transition times.
By assumption, the update distributions are uniform over all actions in each state and satisfy exploring starts.
\autoref{thm: improving policies} thus ensures that the sequence of action values $(q_{\pi_k})_{k \ge 0}$ is componentwise non-decreasing.

Since the policy set is finite, the subsequence $(q_{\pi_{t_n}})_{n \ge 0}$ must settle after finitely many transitions.
In other words, there exists a finite index \(N \ge 0\) such that
$t_N < \infty$ and $t_{N+1} = \infty$.
Moreover, policy $\pi_{t_N}$ must be optimal. Otherwise the modified Robbins-Monro and exploring starts conditions would force $q_k$ to converge to $q_{\pi_{t_N}}$, which in turn would cause the policy to improve and thus $t_{N+1}<\infty$ (as in \autoref{thm: suboptimal blocks finite}).

As a result, for every $k \ge t_N$, the solution of \eqref{eq: deterministic mcopi} satisfies
\[
q_{k+1} = q_k + \alpha_k \mu(\sigma_k, \best{\pi}, f_k)( \qopt - q_k) .
\]
The Robbins-Monro and strict exploring starts conditions then guarantee that $q_k$ converges to $\qopt$ \citep[Proposition~4.1]{Bertsekas1996Neuro-DynamicProgramming}.
% \end{proof}
\hfill \BlackBox

% \section{}
% \label{app:theorem}

% % Note: in this sample, the section number is hard-coded in. Following
% % proper LaTeX conventions, it should properly be coded as a reference:

% %In this appendix we prove the following theorem from
% %Section~\ref{sec:textree-generalization}:

% In this appendix we prove the following theorem from
% Section~6.2:

% \noindent
% {\bf Theorem} {\it Let $u,v,w$ be discrete variables such that $v, w$ do
% not co-occur with $u$ (i.e., $u\neq0\;\Rightarrow \;v=w=0$ in a given
% dataset $\dataset$). Let $N_{v0},N_{w0}$ be the number of data points for
% which $v=0, w=0$ respectively, and let $I_{uv},I_{uw}$ be the
% respective empirical mutual information values based on the sample
% $\dataset$. Then
% \[
% 	N_{v0} \;>\; N_{w0}\;\;\Rightarrow\;\;I_{uv} \;\leq\;I_{uw}
% \]
% with equality only if $u$ is identically 0.} 
% \hfill\BlackBox

% \noindent
% {\bf Proof}. We use the notation:
% \[
% P_v(i) \;=\;\frac{N_v^i}{N},\;\;\;i \neq 0;\;\;\;
% P_{v0}\;\equiv\;P_v(0)\; = \;1 - \sum_{i\neq 0}P_v(i).
% \]
% These values represent the (empirical) probabilities of $v$
% taking value $i\neq 0$ and 0 respectively.  Entropies will be denoted
% by $H$. We aim to show that $\fracpartial{I_{uv}}{P_{v0}} < 0$....\\

% {\noindent \em Remainder omitted in this sample. See http://www.jmlr.org/papers/ for full paper.}

\vskip 0.2in
\bibliography{references}

\end{document}